\documentclass[12pt,a4paper]{amsart}
\usepackage[T1]{fontenc}
\usepackage[utf8]{inputenc}
\usepackage[a4paper,centering,marginpar=25mm]{geometry}
\usepackage[foot]{amsaddr}
\usepackage{amssymb}
\usepackage[hyperfootnotes=false,colorlinks,breaklinks,urlcolor=blue,linkcolor=blue,citecolor=blue]{hyperref}
\usepackage[nameinlink,capitalise]{cleveref}
\usepackage{graphicx}
\usepackage{natbib}
\usepackage{subfig}
\usepackage{xspace}
\usepackage{tikz}
\usepackage{todonotes}
\usepackage[english]{babel}
\newtheorem{theorem}{Theorem}[section]
\newtheorem{lemma}[theorem]{Lemma}
\newtheorem{proposition}[theorem]{Proposition}

\theoremstyle{definition}
\newtheorem{definition}[theorem]{Definition}
\newtheorem{assumption}[theorem]{Assumption}
\theoremstyle{remark}

\newtheorem{remark}[theorem]{Remark}
\numberwithin{equation}{section}

\newcommand{\E}{\mathbb{E}}
\newcommand{\R}{\mathbb{R}}
\newcommand{\N}{\mathbb{N}}

\newcommand{\Fc}{\mathcal{F}}
\newcommand{\Gc}{\mathcal{G}}

\newcommand{\var}{\operatorname{Var}}
\newcommand{\diag}{\operatorname{diag}}
\newcommand{\iid}{\textit{iid}\xspace}
\newcommand{\regimezero}{\emph{balistic}\xspace}
\newcommand{\regimeuno}{\emph{batch-equivalent}\xspace}
\newcommand{\tN}{{\lfloor tN\rfloor}}
\newcommand{\sN}{{\lfloor sN\rfloor}}
\newcommand{\ogamma}{\bar\gamma}

\newcommand{\dDelta}{\Delta^{(d)}}
\newcommand{\cDelta}{\Delta^{(c)}}
\newcommand{\cF}{\mathcal{F}^{(c)}}
\newcommand{\dF}{\mathcal{F}^{(d)}}
\newcommand{\cR}{\mathcal{R}^{(c)}}
\newcommand{\dR}{\mathcal{R}^{(d)}}
\begin{document}

\title[Effective continuous equations for adaptive SGD]{Effective continuous equations for adaptive SGD: a stochastic analysis view}
\author[L. Callisti]{Luca Callisti}
  \address{Dipartimento di Matematica, Universit\`a di Pisa, Largo Bruno Pontecorvo 5, I--56127 Pisa, Italia}
  \email{\href{mailto:l.callisti@studenti.unipi.it}{l.callisti@studenti.unipi.it}}
\author[M. Romito]{Marco Romito}
  \address{Dipartimento di Matematica, Universit\`a di Pisa, Largo Bruno Pontecorvo 5, I--56127 Pisa, Italia}
  \email{\href{mailto:marco.romito@unipi.it}{marco.romito@unipi.it}}
\author[F. Triggiano]{Francesco Triggiano}
  \address{Scuola Normale Superiore, Piazza dei Cavalieri, 7, 56126 Pisa, Italia}
  \email{\href{mailto:francesco.triggiano@sns.it}{francesco.triggiano@sns.it}}
\date{}

\begin{abstract}
We present a theoretical analysis of some popular adaptive Stochastic Gradient Descent (SGD) methods in the small learning rate regime.  Using the stochastic modified equations framework introduced by \cite{li_stochastic_2017}, we derive effective continuous stochastic dynamics for these methods. Our key contribution is that sampling-induced noise in SGD manifests in the limit as independent Brownian motions driving the parameter and gradient second momentum evolutions. Furthermore, extending the approach of \cite{malladi_sdes_2022}, we investigate scaling rules between the learning rate and key hyperparameters in adaptive methods, characterising all non-trivial limiting dynamics.
\end{abstract}
\maketitle

\section{Introduction}

Stochastic Gradient Descent (SGD), dating back to \cite{robbins_stochastic_1951,kiefer_stochastic_1952}, is a fundamental optimization algorithm in machine learning, essential for training large-scale models such as deep neural networks. Its core principle involves updating the model parameters by computing the gradient of a loss function, evaluated on a small, randomly selected, subsets of the data. It turns out that its stochastic nature makes SGD computationally efficient and scalable to massive datasets, forming the bedrock of how many complex models learn.

A crucial parameter in vanilla SGD is the learning rate, that measures the size of the steps taken during the update of the model's parameters. Algorithms like Adagrad (\cite{duchi_adaptive_2011}), RMSprop (\cite{hinton_coursera_2018}), Adam (\cite{kingma_adam_2015}) dynamically adjust the learning rate during training, leading to faster convergence and improved robustness compared to SGD. Due to their practical effectiveness, these adaptive methods are now standard implementations in all major machine learning frameworks (TensorFlow \cite{abadi_tensorflow_2015}, PyTorch \cite{paszke_pytorch_2019}, etc.) and represent the default choice for practitioners. We refer to \cite{ruder_overview_2017} for a review of adaptive algorithms.

It is therefore of paramount importance to have a deep understanding of their internal mechanisms. This understanding is essential not only for effectively applying and tuning these optimizers for specific problems but also for diagnosing training issues, designing novel optimization strategies, and comprehending the theoretical underpinnings of modern machine learning model training.

A significant theoretical advancement in analysing these optimizers has been proposed by the pioneering work of \cite{li_stochastic_2017}, see also \cite{li_stochastic_2019}, by describing the statistical properties of the discrete iterative updates as, within some error, instances of continuous-time stochastic processes, more precisely solutions of stochastic differential equations (SDE). The SDE formalism provides a powerful mathematical framework to study the impact of noise introduced by sampling to understand how optimizers move in complex, non-convex loss landscapes, and potentially reveal implicit regularization effects induced by the optimization process itself.

We exploit the theoretical framework of \emph{Stochastic Modified Equations} (SMEs) of \cite{li_stochastic_2019}, to derive continuous-time effective SDEs for some adaptive SGD algorithms like RMSprop and Adam. It turns out that this is far from being obvious, and uncovers complexities not present before. Indeed, adaptive methods maintain internal state variables (such as exponentially weighted moving averages of squared gradients) which dynamically adapt the learning rate, and that co-evolve with the model parameters, creating coupled dynamics that significantly complicate the derivation of accurate continuous-time approximations.

As one might anticipate from vanilla SGD, the inherent stochasticity arising from sampling random directions introduces a primary noise term into the effective SDE governing the model parameters themselves. The principal contribution of this work lies in identifying and characterizing the emergence of an additional layer of stochasticity that specifically impacts the effective SDE governing the evolution of the internal state variables.

From the perspective of stochastic analysis, the SDE for the squared gradients is driven by the quadratic variation of the primary noise process. More precisely, a newly identified "extra noise" term manifests as the fluctuations of the quadratic variation process. In essence, it captures the higher-order statistical effects of the gradient estimation noise on the adaptive learning rate mechanism itself.

In the course of our study, inspired by \cite{malladi_sdes_2022}, our contribution includes a detailed analysis of the scaling rules between the learning rate and the key hyper-parameters governing adaptive optimization methods (associated with momentum or second-momentum estimation). In particular, we characterise all cases of non-trivial SDE limit behaviours.

\subsection{Related work}

The idea of describing the discrete dynamics of SGD by means of a stochastic differential equation (the so-called stochastic modified equation) has been introduced in the seminal work of \cite{li_stochastic_2017}, and given rigorously in full details in \cite{li_stochastic_2019}. A similar idea appeared in \cite{mandt_stochastic_2017}, but only at a empirical level. Later \cite{gess_stochastic_2024,gess_stochastic_2024b} gave a version, the stochastic modified flow, that while not changing the overall covariance of the noise, allowed to relax the assumptions on the covariance and to study the SGD dynamics consistently as a flow with respect to the initial states.

The SME framework has been later used in a number of different works aiming at giving theoretical explanations of observed empirical facts. For instance, \cite{li_differential_2018} analyse a parallel optimization algorithm, the asynchronous SGD, via a modification of the SME approach, using stochastic differential delay equations. \cite{zhang_stochastic_2024} exploits the SME framework to produce an effective description of the discrete process of dropout. \cite{li_high_2024} proposes a variant (Hessian aware SME) that includes information on the values of the Hessian. \cite{gess_exponential_2024} obtains exponential convergence rates for momentum SGD and for its continuous counterpart, under the \emph{Polyak-{\L}ojasiewicz inequality} assumption (\cite{polyak_gradient_1963}).

The SME framework turns out to be particularly relevant when studying the implicit regularization, or implicit bias, of SGD. For instance, \cite{pesme_implicit_2021} studies the dynamics of SGD for a reparametrization of linear regression through a SDE model, proving that SGD has better generalization properties that standard GD, more precisely SGD performs $\ell^1$ optimization vs $\ell^2$ optimization for GD on a sparse true model. \cite{li_what_2022} performs a thorough analysis of the regularization effects of stochasticity through the limiting SDE. \cite{chen_stochastic_2023,chen_stochastic_2024} show how SGD drives complex networks towards much simpler networks, sparse of low-rank for instance, based on attractivity of some invariant set of the SDE model, that are roughly the same as the discrete SGD dynamics.

When turning to specific limiting models for adaptive methods, we recall \cite{MWE22}, that analyses the training process of adaptive gradient algorithms through a limit ODE. They are able to give a partial explanation of the fast initial convergence of these algorithms, and of the large spikes observed in the late phase of training, due to instabilities around the stationary points. In their setting, the limit is in small learning rate, when the momentum parameters are kept fixed.

An empirical observation, generally noticed several times in the literature, is that adaptive methods converge faster and are more effective in escaping saddle points, but generalise worse than vanilla SGD. This is for instance shown in \cite{xie_adaptive_2022}. Here a SDE model is used, but no continuous dynamics for Adam has been derived. \cite{zhou_towards_2020} carries out an accurate comparison of SGD and ADAM dynamics through the continuous approximation provided by the stochastic modified equations framework. In particular, in \cite{zhou_towards_2020}, noise is modelled through $\alpha$-stable processes and an effective continuous equation that describes the discrete dynamics is obtained only to first order. The effectiveness of SGD and Adam methods has been the subject of an intensive study in \cite{dereich_nonconvergence_2024,dereich_convergence_2024,do_nonconvergence_2025,dereich_averaged_2025,do_nonconvergence_2025b,dereich_sharp_2025}.

\subsection{Limitations and further developments}\label{s:limitations}

Understanding the nature, the magnitude (\cite{mignacco_effective_2022}), the type (\cite{chen_noise_2021}), or the distribution of SGD noise is a topic of current interest. Recently in \cite{simsekli_tail_2019}, see also \cite{wu_noisy_2020,zhou_towards_2020,battash_revisiting_2023}, it has been argued that stochastic gradient noise has heavy tails and should follow $\alpha$-stable Lévy distribution. In view of the role we have identified for the quadratic variation of the stochastic gradient noise, we expect a richer structure due to the Lévy process. We plan to address this issue in more detail in future work. 

It is worth noticing that the order-2 continuous counterparts of the gradient second momentum equations (see equations \eqref{eq:sde2_rmsprop_zero}, \eqref{eq:sde2_rmsprop_uno}, \eqref{eq:sde2_adam_zero}, \eqref{eq:sde2_adam_uno}) need not have non-negative values (with non-negative initial values), although this is the case for the discrete dynamics. This is due to the presence of fluctuations and not necessarily contradictory. See \cref{r:negative} for a more detailed discussion.

We expect that similar phenomena will emerge when turning to the stochastic modified flow framework introduced in \cite{gess_stochastic_2024,gess_stochastic_2024b}, since the latter describes the stochastic gradient noise with a richer structure, without changing the overall covariance. In particular this would allow to relax the regularity assumptions on the covariance (see \cref{r:limitInvert}).

\subsection{Contributions}
In summary we obtain the following main results:
\begin{enumerate}
  \item we identify a one parameter family of continuous SDEs that give a order-1 approximation of adaptive methods, the only one that preserve the memory effect strength and stochasticity of the discrete dynamics, and concentrate on the most relevant values;
  \item we establish that the stochasticity of the sampling mechanism gives rise to an additional, independent random component in order-2 continuous approximations of adaptive optimisation methods.
\end{enumerate}

\subsection{A toy model}\label{s:toymodel}

Before delving into the details of our results, we wish to give an explanation the emergence of additional noise in the effective equations for adaptive SGD. We will do this in a simplified model that emphasises the role of the noise ion the dynamics. Let us thus consider the following problem,
\begin{equation}\label{eq:toy}
  \begin{cases}
    \theta_{k+1}
      = \theta_k - \tau\psi(v_k)g_k,\\
    v_{k+1}
      = v_k - \tau v_k + \tau g_k^2,
  \end{cases}
\end{equation}
which can be seen as a one-dimensional version of \eqref{eq:RMSprop} when $\nabla f(\theta) = 0$, $B=1$. Here $\tau$ will play the role of the time-step. We model the \emph{noise} as a sequence $(\xi_k)_k$ of \iid random variables with
\begin{equation}\label{eq:toy_noise}
  \E[\xi_k] = 0,\qquad
  \var(\xi_k) = 1,\qquad
  \E[\xi_k^3] = 0,\qquad
  \var(\xi_k^2) = \kappa^2,
\end{equation}
and set
\[
  b_k
    = \sqrt\tau\,\xi_k,\qquad
  \omega_k
    = \frac{\sqrt\tau}\kappa(\xi_k^2 - 1).
\]
With these positions, system \eqref{eq:toy} becomes
\begin{equation}
  \begin{cases}
    \theta_{k+1}
      = \theta_k - \sqrt\tau\psi(v_k)m(\theta_k)b_k,\\
    v_{k+1}
      = v_k - \tau (v_k - m(\theta_k)^2) + \kappa\sqrt\tau\,m(\theta_k)^2\omega_k.
  \end{cases}
\end{equation}
It is elementary to see that $b_k$ and $\omega_k$ are uncorrelated, therefore by the Donsker invariance principle (see for instance \cite{karatzas_brownian_1998}) linear interpolations of
\begin{equation}
  B_k
    = \sum_{j=1}^k b_j,\qquad
  W_k
    = \sum_{j=1}^k \omega_j
\end{equation}
(with $B_0=W_0=0)$) converge in law to independent Brownian motions $(B_t,W_t)_{t\geq0}$. In turns, the law of the discrete dynamics \eqref{eq:toy} are well approximated, within order 1, by the SDE
\begin{equation}\label{eq:toy_sde}
  \begin{cases}
    d\theta_t
      = - \sqrt\tau\psi(v_t)m(\theta_t)\,dB_t,\\
    dv_t
      = - (v_t - m(\theta_t)^2)\,dt + \kappa\sqrt\tau\,m(\theta_t)^2\,dW_t.
  \end{cases}
\end{equation}
Let us make more precise the connection of the emergence of the additional randomness with the notion of quadratic variation. Define the $\sigma$-fields $\Fc_k=\sigma(\theta_0,v_0,\xi_1,\xi_2,\ldots,\xi_k)$, then the processes $(B_k)_{k\geq0}$, $(W_k)_{k\geq0}$ are martingales with respect to the filtration $(\Fc_k)_{k\geq0}$. Each has (discrete) quadratic variation $k\tau$, and their (discrete) quadratic \emph{co-variation} is
\[
  \E[(B_{k+1}W_{k+1} - B_kW_k\mid\Fc_k]
    = \E[(B_{k+1} - B_k)(W_{k+1} - W_k)\mid\Fc_k]
    = 0.
\]

\begin{remark}[Beyond the low skewness condition]\label{r:beyond}
  The assumption $\E[\xi_k^3] = 0$, in our simplification, corresponds to the \emph{low skewness condition} introduced in \cite{malladi_sdes_2022} and stated in \cref{LSC}. Let us assume that $\E[\xi_k^3] = \gamma>0$. This would introduces correlations between the two noises, so to recover independence we set
  \[
    \omega_k
      = \frac{\sqrt\tau}{\sqrt{\kappa^2 - \gamma^2}}(\xi_k^2 - \gamma\xi_k - 1).
  \]
  Now, once again, $b_k$ and $\omega_k$ are uncorrelated and $B$, $W$ converge to independent Brownian motions. The corresponding SDE is now,
  \begin{equation}\label{eq:toy_sde2}
    \begin{cases}
      d\theta_t
        = - \sqrt\tau\psi(v_t)m(\theta_t)\,dB_t,\\
      dv_t
        = - (v_t - m(\theta_t)^2)\,dt
          +\gamma m(\theta_t)^2\,dB_t
          + \sqrt{\tau(\kappa^2 - \gamma^2)}m(\theta_t)^2\,dW_t.
    \end{cases}
  \end{equation}
  Notice that in particular the quadratic variation of the stochastic integral in the equation for $v$ in \eqref{eq:toy_sde} is the same as the quadratic variation of the same term in the same equation in \eqref{eq:toy_sde2}. Therefore, the overall effect of a non-vanishing skewness is essentially the redistribution of the randomness in the equation for the \emph{second momentum}. In the rest of the paper we shall rely on the low skewness condition (and of even stronger assumptions), but one can expect that a non-trivial skewness will produce similar effects in the (much more involved) effective continuous SDEs for adaptive SGD.
\end{remark}

\section{Preliminaries}

Let $(x_j)_{j=1}^M$ be a (labelled) dataset and $B$ be the batch size. Consider a loss function $f(x,\theta)$ and a family $(\gamma_{k,i})_{k\in \N,i\le B}$ of \iid random variables with uniform law over $\{1,\dots,M\}$.

The parameter update at step $k$ of any stochastic optimization algorithm depends on
\begin{equation}\label{eq: gradDecomposition}
    g_k(\theta)=\nabla f(\theta)+\frac{1}{B}\sum_{i=1}^B(\nabla  f_{\gamma_{k,i}}(\theta)-\nabla f(\theta))=:\nabla f(\theta)+\frac{z_k}{\sqrt{\tau B}},
\end{equation}
where $\nabla f_{\gamma_{k,i}}(\theta):=\nabla_{\theta}f (x_{\gamma_{k,i}},\theta)$, $\nabla f(\theta):=\E[\nabla f_{\gamma_{k,1}}(\theta)]$. The parameter $\tau>0$ denotes the time step of the algorithm, we will show that it depends on the scaling rule between the algorithm's hyper-parameters. Instead, $(z_k)_k$ is a family of independent centred random variables with covariance $\tau\Sigma(\theta)$, where $\Sigma(\theta)$ is given by
\begin{equation}\label{eq:covariance}
  \Sigma(\theta)
    := \E[(\nabla f_{\gamma_{k,1}}(\theta)-\nabla f(\theta))(\nabla f_{\gamma_{k,1}}(\theta)-\nabla f(\theta))^T].
\end{equation}
In the following, its diagonal is denoted by $\Sigma_d(\theta)$. Instead, $M(\theta)$ denotes the covariance matrix of $\hat{z_k}=(\tau^{-1/2} z_k)^{\odot2}$, namely
\begin{equation}\label{eq:M}
 M(\theta)=\E\left[\left((\hat{z}_k-\Sigma_d(\theta)\right)\left((\hat{z}_k-\Sigma_d(\theta)\right)^T\right].   
\end{equation}

\begin{definition}[RMSprop]\label{d:RMSprop}
The \emph{RMSprop} algorithm updates $\theta_k$ as follows
    \begin{equation}\label{eq:RMSprop}
      \begin{cases}
        \theta_{k+1}=\theta_{k}-\eta g_k \odot (\sqrt{v_k}+\epsilon)^{-1},\\
        v_{k+1}=\beta v_k+(1-\beta)g_k^{\odot 2},
      \end{cases}
    \end{equation}
    where $\eta,\beta$ and $\epsilon$ are positive hyperparameters.
\end{definition}

\begin{definition}[Adam]\label{d:Adam}
The \emph{Adam} algorithm updates $\theta_k$ as follows
  \begin{equation}\label{eq:Adam}
    \begin{aligned}
    \theta_{k+1}
      &=\theta_{k}-\eta \widehat{m}_{k+1} \odot (\sqrt{\widehat v_k}+\epsilon)^{-1},\\
    m_{k+1}
      &=\beta_1 m_k+(1-\beta_1)g_k, &
    \widehat{m}_{k+1}
      &= (1-\beta_1^{k+1})^{-1}m_{k+1},\\
    v_{k+1}
      &= \beta_2 v_k+(1-\beta_2)g_k^{\odot 2}, &
    \widehat{v}_{k+1}
      &= (1-\beta_2^{k+1})^{-1}v_{k+1}.
    \end{aligned}
  \end{equation}
  where $\eta,\beta_1,\beta_2$ and $\epsilon$ are positive hyper-parameters.
\end{definition}

\begin{remark}\label{r:Remark on phi}
  The term $\sqrt{v_k}$ in formulas above is not ``smooth''. To prove rigorously our results, \emph{without changing the discrete dynamics}, we will introduce a regularization of the square root that is relevant only in a small neighbourhood of the value $v=0$ which is not reached by the momentum variables in the algorithms above. This has an additional benefit, see \cref{r:negative}.
\end{remark}

In the definitions above and in the rest of the paper, the operator $\odot$ is the component-wise product of two vectors. Moreover, we overload the meaning of $\odot$ when applied to diagonal matrices to mean the $\odot$ product applied to the diagonal. In addition, whenever a scalar function is computed at a vector or matrix, then it is applied component-wise. At last, given a vector $v$, then $v^{\otimes2}$ and $\diag(v)$ denote $v\otimes v$ and a diagonal matrix with $v$ on the diagonal, respectively.

Let us conclude this preliminary section by introducing the concept of approximation and some assumptions on the noise. To this end denote by $\Gc^\alpha$ the space of $C^{\alpha +1}$-functions which, together with all the derivatives up to order $\alpha$, have polynomial growth (\cite{li_stochastic_2019}).

\begin{definition}\label{defn:approx}
    Let $T>0, \tau \in (0,1\wedge T)$ and $\alpha\in \N^*$. Set $N:=\lfloor T/\tau \rfloor$. A continuous-in-time stochastic process $(X_t)_{t \in [0,T]}$ is an $\alpha$ order weak approximation of the discrete-in-time stochastic process $(x_k)_{k=0}^N$ if there exists $C>0$ such that
    \[
      \sup_{k=0,\dots,N}|\E[f(X_{k\tau})-f(x_k)]|\le C\tau^{\alpha}
    \]
    for all $f:\R^d\to\R$ in $\Gc^{\alpha +1}$.
\end{definition}

In the following, we assume that the following conditions holds:
\begin{assumption}{(Low-Skewness Condition)}\label{LSC}
For all $k$, $z_k$'s third moments are controlled by $\tau^p$ for some $p\geq 3$. Namely, if $\theta\in \R^d$, then there exist $K(\theta) \in \Gc^0$, 
    \begin{equation}
        |\E[(z_k)_{i_1}(z_k)_{i_2}(z_k)_{i_3}]|\le K(\theta)\tau^p,
    \end{equation}
    for all $i_1,i_2,i_3\in \{1,\dots,d\}$.
\end{assumption}

\begin{remark}
   As already stated in \cref{s:toymodel}, it seems that this assumption might be weakened. Nevertheless, it will allow us to obtain more treatable SDEs without causing any loss of approximation accuracy in the experiments. Moreover, as pointed out in \cite{malladi_sdes_2022}, a similar condition is trivially satisfied whenever the $z_k$ are Gaussian which is a reasonable assumption since we are looking for SDEs driven by Brownian motion.
\end{remark}

\begin{assumption}{(Bounded Moments Condition)}\label{BMC}
    For all $m\in \N$ there exists $C_m$ such that
    \[
      \sup_{k}\E[|z_k|^{2m}]
        \le C_m(1+|\theta|^{2m}).
    \]
\end{assumption}

\section{Scaling rules: a heuristic derivation}\label{s:scaling}

Adaptive algorithms are characterized by various hyper-parameters: the learning rate, the momentum factors (denoted by $\beta$ in \eqref{eq:RMSprop} and $\beta_1$, $\beta_2$ in \eqref{eq:Adam}), and the batch size. Imposing a scaling rule between the learning rate $\eta$ and the momentum factors is known to preserve the memory effect strength (\cite{MWE22}). A scaling rule between the learning rate and batch size helps in dynamically adjusting these quantities (\cite{GZR22,malladi_sdes_2022}). We thus consider the problem under the assumption that the momentum factors $\beta$, $\beta_1$, $\beta_2$ converge to $1$ as long as $\eta \to 0$. Using formal stochastic analysis arguments, we shall derive all the possible scaling rules preserving the memory effect and stochasticity.

We first focus on RMSprop. Let 
\[
  1-\beta\simeq\eta^b,\qquad
  v_k=\eta^{-2c}u_k,\qquad
  \epsilon=\eta^{-c}\varepsilon,\qquad
  B\simeq\eta^d,
\]
where $x\simeq h(\eta)$ means that there exists $C$ independent from $\eta$ such that $x=C\,h(\eta)$. Then, \eqref{eq:RMSprop} can be reformulated as follows,
\begin{equation}\label{Update:RMS}
  \begin{cases}
    \theta_{k+1}
      = \theta_{k}-\eta^{1+c}(\sqrt{u_k}+\varepsilon )^{-1}\odot(\nabla f(\theta_k)+(\tau \eta^{d})^{-\frac{1}{2}}z_k),\\
    u_{k+1}
      = u_k-\eta^bu_k+\eta^{b+2c}(\nabla f(\theta_k)+(\tau \eta^{d})^{-1/2}z_k)^{\odot2}.
  \end{cases}
\end{equation}

Since we want to preserve the memory effect, that is the damping term $-\eta^bu_k$, and the gradient update, namely $\frac{\nabla f(\theta_k)}{ \sqrt{u_k}+\varepsilon}$, then 
\[
  \tau=\eta^{b}=\eta^{1+c}
    \quad\implies\quad b=1+c.
\]
Let us rewrite the remaining terms in integral form,
\[
    \sum_{k=0}^N\eta^{2c-d}z_k^{\odot 2},\quad
    \sum_{k=0}^N\eta^{2c+\frac{b-d}{2}}\nabla f(\theta_k)\odot z_k,\quad
    \sum_{k=0}^N\tau \eta^{2c}\nabla f(\theta_k)^{\odot 2},\quad
    \sum_{k=0}^N\eta^{\frac12(1+c-d)}z_k\odot(\sqrt{u_k}+\varepsilon )^{-1}.
\]
Then, additional conditions on $b,c$ and $d$ follow from requiring that, as long as $\eta$ goes to 0, the following conditions hold.
\begin{itemize}
  \item Every term vanishes or converges to some integral;
  \item the adaptive parameter, $v$, is updated by the square of the stochastic gradient, that is at least one of the first terms converges;
  \item the stochasticity has an effect, namely at least one of the three terms involving $z_k$ or $z_k^2$ \footnote{Recall that under the assumption that the $z_k$'s are Gaussian increments, $\sum_kz_k^2$ converges, heuristically, to a deterministic integral, the quadratic variation} converges.
\end{itemize}
The above requirements impose that $c\in[0,1]$ and $d=2c$.

Likewise, for Adam \eqref{eq:Adam}, take $1-\beta_1\simeq\eta^b$, $\beta_2\simeq\eta^b$, to keep the memory effect in both equations for first and second momentum, and $\tau=\eta^b$. Next, let
\[
  v_k
    \simeq \eta^{-2c}u_k,\qquad
  \epsilon
    \simeq \eta^{-c}\varepsilon,\qquad
  B
    \simeq \eta^d,\qquad
  m_k
    \simeq \eta^{-e}\ell_k.
\]
The equations for the first momentum is
\[
    \ell_{k+1}
      = \ell_k - \eta^b\ell_k
        + \eta^{b+e}\nabla f(\theta_k)
        + \eta^{\frac12(b+2e-d)}z_k,\\
\]
Thus $e=0$ to keep the gradient in the equation for the first momentum, therefore $b=1+c$ (as in the case of RMSprop), and $b-d\geq0$ to avoid explosion of the second term in the same equation. From now on the same considerations as the above RMSprop case apply, and again $c\in[0,1]$ and $d=2c$.

In the rest of the paper we shall focus on the scaling rules given by the endpoints $c=0$ and $c=1$. The value $c=0$, which we shall call \regimezero regime, is theoretically interesting because it makes it possible to preserve the term $\nabla f(\theta)^2$, which is the main idea behind adaptive algorithms. Instead, the value $c=1$, which we shall call \regimeuno regime, is experimentally relevant given its coherence with the square root scaling rule between $\eta$ and $B$ identified and analysed in \cite{GZR22,malladi_sdes_2022}. 

We first identify the limit SDEs (at first order in the time step) for both models and both regimes. Consider first the \regimezero regime, and set,
\begin{equation}\label{eq:values_zero}
1 - \beta
  = \lambda_0\eta,\qquad
1 - \beta_1
  = \lambda_1\eta,\qquad
1 - \beta_2
  = \lambda_2\eta,\qquad
B
  = \sigma^{-2},
\end{equation} 
\begin{equation}\label{eq:gammas}
    \gamma_1(t) = (1 - e^{-(t+\tau)\sum_{k=1}^5(\lambda_1^k\tau^{k-1})/k}),\quad \gamma_2(t) = (1 - e^{-t\sum_{k=1}^5(\lambda_2^k\tau^{k-1})/k}).
\end{equation} We also set $u_k=v_k$, for consistency with the other regime.
The order-1 SDEs for RMSprop and Adam are as follows,
\begin{equation}\label{eq:sde1_rmsprop_zero}
  \textmd{\footnotesize(RMSprop)}\qquad
  \begin{cases}
    d\theta_t
      = - \nabla f(\theta_t)\odot(\sqrt{u_t} + \epsilon)^{-1}\,dt,\\
    du_t
      = \lambda_0(\nabla f(\theta_t)^{\odot 2} + \sigma^2\Sigma_d(\theta_t) - u_t)\,dt.
  \end{cases}
\end{equation}
\begin{equation}\label{eq:sde1_adam_zero}
  \textmd{\footnotesize(Adam)}\qquad
  \begin{cases}
    d\theta_t
      = -\frac{\sqrt{\gamma_2(t)}}{\gamma_1(t)}m_t\odot(\sqrt{u_t} + \epsilon\sqrt{\gamma_2(t)})^{-1}\,dt,\\
    dm_t
      = \lambda_1 (\nabla f(\theta_t) - m_t)\,dt,\\
    du_t
      = \lambda_2(\sigma^2\Sigma_d(\theta_t) + \nabla f(\theta_t)^{\odot 2} - u_t)\,dt.
  \end{cases}
\end{equation}

We turn to the \regimeuno regime. Let $(W_t)_t$ be a $d$-dimensional Brownian motion, where $d$ is the dimension of the space of parameters $\theta$.  Set
\begin{equation}\label{eq:values_uno}
  1 - \beta
    = \lambda_0\eta^2,\qquad
  1 - \beta_1
    = \lambda_1\eta^2,\qquad
  1 - \beta_2
    = \lambda_2\eta^2,\qquad
  \epsilon
    = \eta^{-1}\epsilon_0,\qquad
  B
    = \sigma^{-2}\eta^2,
\end{equation}
and $u_k=\eta^2 v_k$. The order-1 SDE for RMSprop \eqref{eq:RMSprop} and Adam \eqref{eq:Adam} are as follows,
\begin{equation}\label{eq:sde1_rmsprop_uno}
  \textmd{\footnotesize(RMSprop)}\qquad
  \begin{cases}
    d\theta_t
      = -(\sqrt{u_t} + \epsilon_0)^{-1}\odot \bigl(\nabla f(\theta_t)\,dt + \sigma\Sigma(\theta_t)^{\frac12}\,dW_t\bigr),\\
    du_t
      = \lambda_0(\sigma^2\Sigma_d(\theta_t) - u_t)\,dt.
  \end{cases}
\end{equation}
\begin{equation}\label{eq:sde1_adam_uno}
  \textmd{\footnotesize(Adam)}\qquad
  \begin{cases}
    d\theta_t
      = -\frac{\sqrt{\gamma_2(t)}}{\gamma_1(t)}m_t\odot(\sqrt{u_t} + \epsilon_0\sqrt{\gamma_2(t)})^{-1}\,dt,\\
    dm_t
      = \lambda_1 (\nabla f(\theta_t) - m_t)\,dt
        + \lambda_1\sigma\Sigma(\theta_t)^{\frac12}\,dW_t,\\
    du_t
      = \lambda_2(\sigma^2\Sigma_d(\theta_t) - u_t)\,dt.
  \end{cases}
\end{equation}

We summarise our findings in the following theorem. We point out that the results for the \regimeuno case can be originally found in \cite{malladi_sdes_2022}.

\begin{theorem}\label{t:scaling}
  Assume \cref{LSC}, \cref{BMC} and that $f$, $\Sigma^{\frac12}$ are sufficiently regular. Let $\tau=\eta$ in the \regimezero regime, and $\tau=\eta^2$ in the \regimeuno regime. Let $T>0$, then $(X_t)_t$ is a order-1 weak approximation of $(x_k)_k$ (according to \cref{defn:approx}),
  where either $x_k=(\theta_k,u_k)$ is the rescaled dynamics of RMSprop (in either regime) and $X_t=(\theta_t,u_t)$ is solution of respectively \eqref{eq:sde1_rmsprop_zero} or \eqref{eq:sde1_rmsprop_uno}, or $x_k=(\theta_k,\ell_k,u_k)$ is the rescaled dynamics of Adam (in either regime) and $X_t=(\theta_t,\ell_t,u_t)$ is solution of respectively \eqref{eq:sde1_adam_zero} or \eqref{eq:sde1_adam_uno}.
\end{theorem}

\section{Effective continuous equations for adaptive algorithms}

In this section we turn to investigate order-2 effective continuous equations for adaptive optimization methods within the \regimezero and \regimeuno regimes, thereby extending the findings of \cite{malladi_sdes_2022}. In contrast to their order-1 counterparts, these higher-order models incorporate supplementary terms arising from It\=o corrections. The inclusion of these corrections is crucial for mitigating discrepancies at finer scales and enhancing the order of approximation of the discrete optimization dynamics.

A significant departure from the order-1 effective models discussed in \cref{s:scaling} is observed in the order-2 representations for adaptive methods. Specifically, our analysis reveals a novel phenomenon: the emergence of an additional, independent random component within the governing equation for the second momentum. This emergent randomness is intrinsically linked to, and generated by, the fluctuations of squared magnitudes of the gradient estimators. It is worth noticing that the underlying SGD process possesses a singular source of randomness, namely the mini-batch sampling mechanism. The manifestation of independent noise components in the continuous dynamics of adaptive SGD arises because the equations governing these dynamics capture two distinct types of fluctuations, as illustrated in the elementary example of \cref{s:toymodel}.

Consider first the \regimezero regime. Set for simplicity $\lambda_0=\lambda_1=\lambda_2=1$ in \eqref{eq:values_zero}, and $u_k=v_k$ again for consistency with the other regime. The order-2 SDEs for RMSprop and Adam are given, respectively, as follows,
\begin{equation}\label{eq:sde2_rmsprop_zero}
  \begin{cases}
    \begin{aligned}
      d\theta_t
        &= - P_t\odot \nabla f(\theta_t)
           -\frac12\tau\bigl(P_t^{\otimes 2}\nabla^2 f(\theta_t)\nabla f(\theta_t)
           + \frac12P_t^{\otimes 2}u_t^{-\frac12}\odot(J_t-u_t)\odot \nabla f(\theta_t)\bigr)\,dt\\
        &\quad + \sqrt{\tau}\sigma \diag(P_t)\Sigma(\theta_t)^{1/2}\,dW_t,
    \end{aligned}\\
    \begin{aligned}
      du_t
        &=  \bigl(1+\tfrac12\tau\bigr) (J_t - u_t)
         + \frac12\tau\nabla J_t(\nabla f(\theta_t)\odot P_t)\\
         &\quad + -2\sqrt{\tau}\sigma\diag(\nabla f(\theta_t))\Sigma(\theta_t)^{1/2}\,dW_t
         + \sqrt\tau\sigma^2 M(\theta_t)^{1/2}\,dB_t,
    \end{aligned}
  \end{cases}
\end{equation}

\begin{equation}\label{eq:sde2_adam_zero}
  \begin{cases}
    \begin{aligned}
      d\theta_t
        & = - \Gamma_t m_t\odot Q_t
            + \frac12\tau\bigl(\Gamma_t(\nabla f(\theta_t)-m_t)\odot Q_t
            + m_t\odot (\partial_t(\Gamma_t Q_t(\cdot))(u_t))\\
        &\quad + \frac12\Gamma_tQ_t^{\otimes2}u_t^{-\frac12} \odot m_t\odot (J_t - u_t)\bigl)\,dt,
    \end{aligned}\\
      dm_t
        = \bigl(1+\tfrac12\tau\bigr)(\nabla f(\theta_t)-m_t)
        + \frac12\tau\Gamma_t\bigl(\nabla^2f(\theta_t)(m_t\odot Q_t)\bigr)\,dt
        + \sqrt{\tau}\sigma\Sigma(\theta_t)^{1/2}\,dW_t,\\
    \begin{aligned}
      du_t
        & = \bigl(1+\tfrac12\tau\bigl)(J_t-u_t)
          + \frac12\tau\bigl(\nabla J_t(m_t\odot \Gamma_t Q_t\bigr)\,dt\\
        & \quad - 2\sqrt{\tau}\sigma\diag(\nabla f(\theta_t))\Sigma(\theta_t)^{1/2}\,dW_t
          + \sigma^2 M(\theta_t)^{1/2}\,dB_t,
    \end{aligned}
  \end{cases}
\end{equation}
where $M$ is defined in \eqref{eq:M}, $\Gamma_t=\sqrt{\gamma_2(t)}/\gamma_1(t)$, with $\gamma_i$ given in \cref{eq:gammas},
\begin{equation}\label{eq:positions}
  P_t
    = (\sqrt{u_t} + \epsilon)^{-1},\qquad
  J_t
    = \nabla f(\theta_t)^{\otimes 2} + \sigma^2\diag\Sigma(\theta_t),\qquad
  Q_t
    = \bigl(\sqrt{u_t} + \epsilon\sqrt{\gamma_2(t)}\bigr)^{-1},
\end{equation}
and $(B_t)_{t\geq0}$, $(W_t)_{t\geq0}$ are two \emph{independent} $d$-dimensional Brownian motions.

In the \regimeuno regime, we set again for simplicity $\lambda_0=\lambda_1=\lambda_2=1$, $\epsilon_0=\epsilon$ in \eqref{eq:values_uno}, and $u_k=\eta^2 v_k$. The order-2 SDEs for RMSprop and Adam are given, respectively, as follows,
\begin{equation}\label{eq:sde2_rmsprop_uno}
  \begin{cases}
    \begin{aligned}
      d\theta_t
        &= - P_t\odot\nabla f(\theta_t)
           - \frac12\tau\bigl(P_t^{\otimes 2}\nabla^2 f(\theta_t)\nabla f(\theta_t)
           + \frac12P_t^{\otimes2}u_t^{-\frac12}\odot\nabla f(\theta_t)\bigr)\\
        &\quad -\tau\sigma^2\sum_{h,k=1}^d({P_t}^{\otimes 2}_{h,k}\Sigma_{h,k}(\theta_t)\nabla(\partial^2_{h,k}f(\theta_t))\odot P_t\,dt\\
        &\quad + \bigl(\sigma\diag(P_t)\Sigma(\theta_t)^{1/2}
           + \tau\Lambda_1(\theta_t,u_t)\bigr)\,dW_t,
    \end{aligned}\\
    \begin{aligned}
      du_t
        &= \bigl(1+\tfrac12\tau\bigr)(\sigma^2\Sigma_d(\theta_t)-u_t)
           + \tau\bigl(\nabla f(\theta_t)^{\otimes2}-\frac12\sigma^2\nabla \Sigma_d(\theta_t)(\nabla f(\theta_t)\odot P_t\bigr)\\
        &\quad - \tau\sigma^4 \sum_{h,k=1}^d {P_t}^{\otimes 2}_{h,k}\Sigma_{h,k}\partial^2_{h,k}\Sigma_d(\theta_t)\,dt
          + \tau\Lambda_2(\theta_t,u_t)\,dW_t
          + \sqrt{\tau}M(\theta_t)^{1/2}\,dB_t.
    \end{aligned}
  \end{cases}
\end{equation}
\begin{equation}\label{eq:sde2_adam_uno}
  \begin{cases}
    \begin{aligned}
      d\theta_t
        &= - \Gamma_t m_t\odot Q_t
           + \tau \bigl(\tfrac32(\Gamma_t\nabla f(\theta_t)-m_t)\odot Q_t
           + m_t\odot (\partial_t(\Gamma_t Q_t(\cdot))(u_t))\\
        &\quad + \frac12\Gamma_t Q_t^{\otimes2}u_t^{\frac12}\odot(\sigma^2\Sigma_d(\theta_t)-u_t)\bigr)\,dt
           + \frac12\tau\sigma\Gamma_t\diag(Q_t)\Sigma(\theta_t)^{1/2}\,dW_t,
    \end{aligned}\\
    \begin{aligned}
      dm_t
        &=  \bigl(1+\tfrac12\tau\bigr)(\nabla f(\theta_t)-m_t)
          + \frac12\tau \Gamma_t\nabla^2f(\theta_t)(m_t\odot Q_t)\,dt\\
        &\quad + \sigma\Bigl(\Sigma(\theta_t)^{1/2}
          + \frac12\tau\sum_{h=1}^d \Gamma_t(m_t\odot Q_t)_h\partial_h\Sigma(\theta_t)^{1/2}\Bigr)\,dW_t,
    \end{aligned}\\
    \begin{aligned}
      du_t
        &= \bigl(1-\tfrac12\tau\bigr)(\sigma^2\Sigma_d(\theta_t)-u_t)\,dt 
         + \frac12\tau\sigma^2\Gamma_t\nabla\Sigma_d(\theta_t)(m_t\odot Q_t)\,dt\\
        &\quad -2\tau\sigma\diag(\nabla f(\theta_t))\Sigma(\theta_t)^{1/2}\,dW_t
         + \sqrt{\tau}M(\theta_t)^{1/2}\,dB_t.
    \end{aligned}
  \end{cases}
\end{equation}
The terms $\Lambda_1$, $\Lambda_2$ in \eqref{eq:sde2_rmsprop_uno} are explicitly given in \cref{as:order2} because of their involved expressions.
\begin{remark}\label{r:limitInvert}
  In general, $\Lambda_1$ is well-defined only if $\Sigma(\theta_t)^{1/2}$ is invertible, although the assumption can be dropped if $\Sigma$ is diagonal. This is the least reasonable among all the hypotheses on $f$, $\Sigma$, $M$ used in this work and, for example, does not hold in over-parametrized neural networks. 
\end{remark}

We can now state our main theorem. The full set of assumptions and a rigorous proof will be given in \cref{as:order2}.
\begin{theorem}
  Assume \cref{LSC}, \cref{BMC}, that $f$, $\Sigma^{\frac12}$ and $M^{\frac12}$ are sufficiently regular, and $\Sigma^{\frac12}$ invertible (for RMSprop). Let $\tau=\eta$ in the \regimezero regime, and $\tau=\eta^2$ in the \regimeuno regime.
  Let $T>0$, then $(X_t)_t$ is a order-2 weak approximation of $(x_k)_k$ (according to \cref{defn:approx}),
  where either $x_k=(\theta_k,u_k)$ is the rescaled dynamics of RMSprop (in either regime) and $X_t=(\theta_t,u_t)$ is solution of respectively \eqref{eq:sde2_rmsprop_zero} or \eqref{eq:sde2_rmsprop_uno}, or $x_k=(\theta_k,\ell_k,u_k)$ is the rescaled dynamics of Adam (in either regime) and $X_t=(\theta_t,\ell_t,u_t)$ is solution of respectively \eqref{eq:sde2_adam_zero} or \eqref{eq:sde2_adam_uno}.
\end{theorem}

\begin{remark}\label{r:negative}
  It is worth noting that the $u$ components of the order-2 continuous equations obtained in this section are not guaranteed to be non-negative, despite what one might expect from the discrete dynamics they approximate.
  However, there is no contradiction in this result. For example, the SDE $d\hat x = \hat x\,dt + \sqrt\eta\,dB_t$ is a bona fide order-1 approximation of the (continuous and non-random) dynamics $\dot x = x$. The variable $\hat x$ has variance of order $O(\eta)$, and while negative values of $\hat x$ are possible, they become increasingly unlikely as $\eta\to0$.
  For this reason in the proofs given in \cref{as:order2} we shall modify the term $\sqrt u$ to ensure that it is always well defined.
\end{remark}

\begin{table}
  \centering
  \begin{tabular}{cc}
    \includegraphics[width=.45\linewidth]{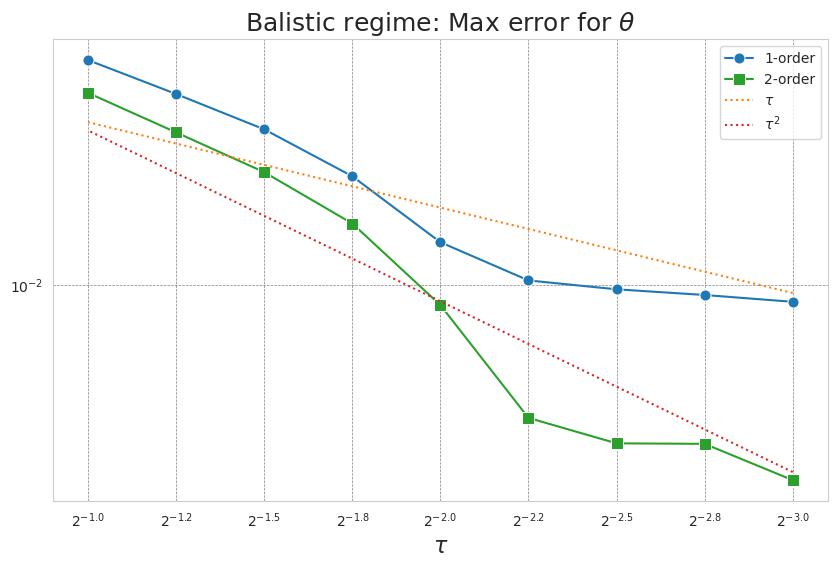} &
    \includegraphics[width=.45\linewidth]{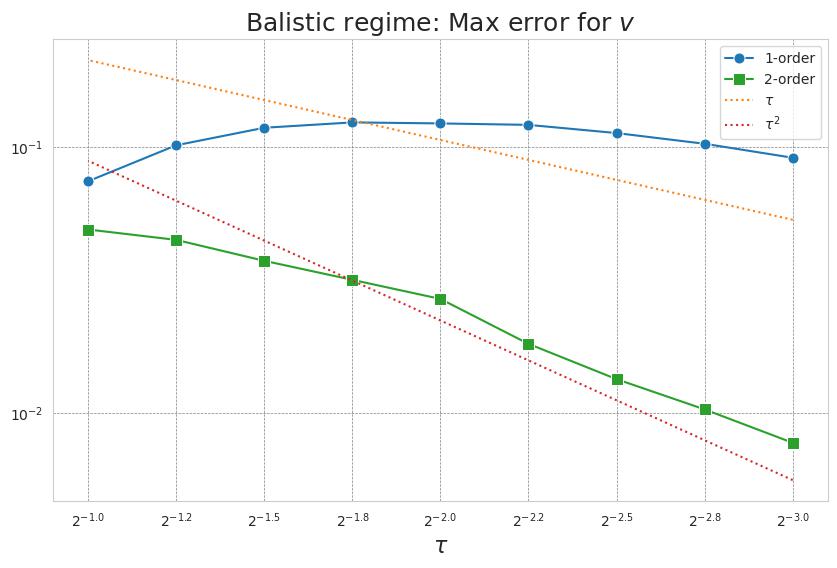} 
  \end{tabular}
  \caption{\textit{RMSprop} \regimezero regime: We show the weak error, see \cref{defn:approx}, with two test functions $f_1(\theta)=\frac{\|\theta\|_2^2}{d}$ and $f_2(v)=\frac{\|v\|_2^2}{d},$
  where $d$ denotes the dimension of $\theta$, in our case $d=6$. As predicted
by our analysis, the order-2 approximation  should give a slope = 2 
decrease in error as $\tau$ decreases (note that the x-axis is flipped).}\label{tb:zero}
\end{table}
\begin{table}
  \centering
  \begin{tabular}{cc}
    \includegraphics[width=.45\linewidth]{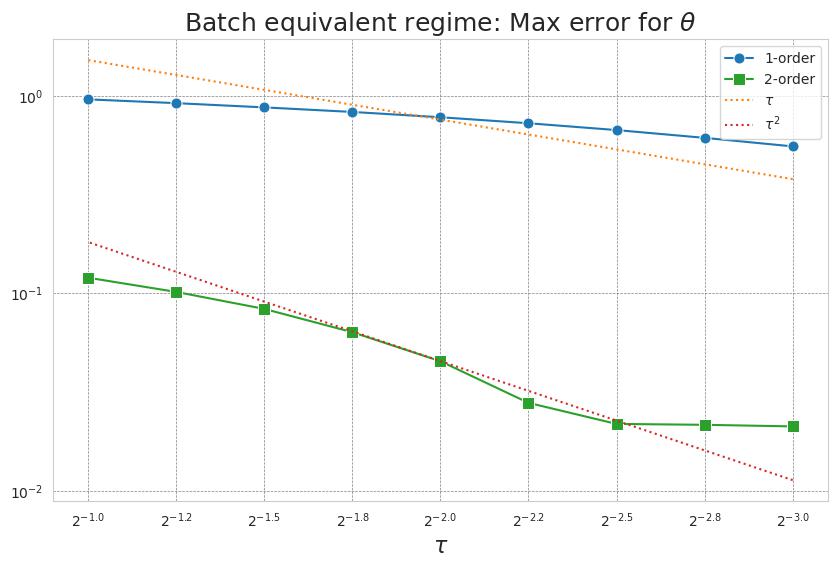} &
    \includegraphics[width=.45\linewidth]{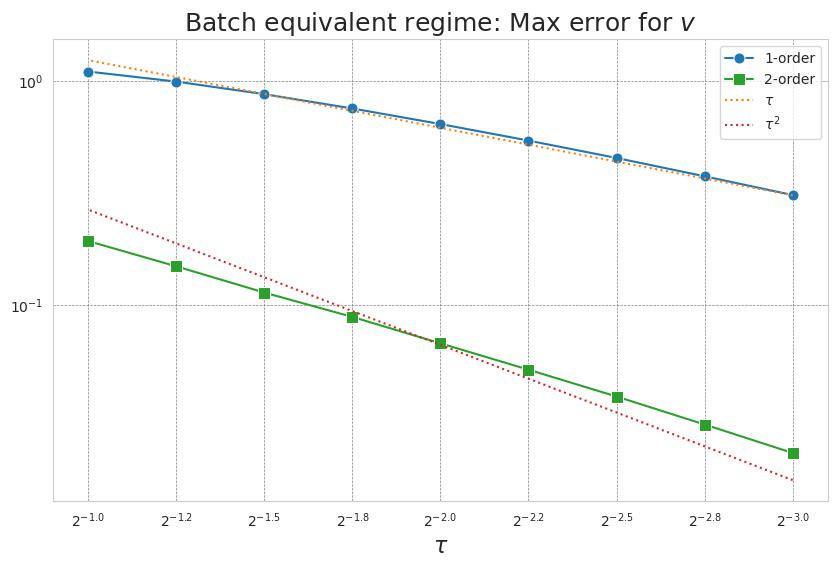} 
  \end{tabular}
  \caption{\textit{RMSprop} \regimeuno regime:  We show the weak error, see \cref{defn:approx}, with two test functions $f_1(\theta)=\frac{\|\theta\|_2^2}{d}$ and $f_2(v)=\frac{\|v\|_2^2}{d},$
  where $d$ denotes the dimension of $\theta$, in our case $d=6$.}\label{tb:uno}
\end{table}

The numerical results in \cref{tb:zero} (\regimezero regime) and \cref{tb:uno} (\regimeuno regime) have been obtained on RMSprop, with $\lambda_0=0.5$ and $\sigma=1$, using a synthetic Gaussian dataset and quadratic loss. More details on the experimental setting and a larger set of experiments are available in the appendix (see \cref{as:experiments}).

\section{Appendix}
\subsection{Details on the toy model}\label{as:toymodel}

In this section we provide additional details on the toy model presented in \cref{s:toymodel}.

This section analyses a toy model, a simplified version of \eqref{Update:RMS}. Consider
\begin{equation}\label{eq:System1}
  \begin{cases}
    \theta_{k+1}
      = \theta_k-\tau\psi(v_k)g_k,\\
    v_{k+1}
      = v_k-\tau v_k+\tau g_k^2,
  \end{cases}
\end{equation}
where $\theta_k,v_k\in \R$ and $g_k$ is given by \eqref{eq: gradDecomposition}. Since we aim to approximate this discrete evolution by an SDE, we need to identify the stochastic drivers most suited to describe the random variable $(z_k,z_k^2)$. Let us introduce the following assumptions:
\begin{equation}
    \nabla f(\theta)\equiv 0,\qquad
    B = 1,\qquad
    z_k = m(\theta_k)\sqrt{\tau}\xi_k,
\end{equation}
where $(\xi_k)_{k\geq0}$ is a sequence of \iid random variables such that 
\begin{equation}\label{eq:moments}
  \E[\xi_k]=0,\qquad
  \var(\xi_k^2) = 1,\qquad 
  |\E[\xi_k^3]|\leq\tau,\qquad
  \var(\xi_k^2) = \kappa^2,\qquad
  \sup_k\E[\xi_k^6] = \overline{c},
\end{equation}
where the constraint on the third moment corresponds to \cref{LSC}, while $\kappa,\overline{c}\in\R$ are constants independent of $\tau$.

Therefore, \eqref{eq:System1} corresponds to
\begin{equation}\label{eq:System2}
  \begin{cases}
    \theta_{k+1}
      = \theta_k - \sqrt{\tau}\psi(v_k)m(\theta_k)b_k,\\
    v_{k+1}
      = v_k - \tau(v_k-m(\theta_k)^2) + \kappa\sqrt\tau m(\theta_k)^2 \omega_k,
  \end{cases}
\end{equation}
where 
\begin{equation}\label{eq:2Noises}
  b_k
    := \sqrt{\tau}\xi_k,\qquad
  \omega_k
    := \frac{\sqrt\tau}\kappa(\xi_k^2-1).
\end{equation}
Let us highlight that the integral formulation of the previous dynamic is given by
\begin{equation}\label{eq:System2Int}
  \begin{cases}
    \theta_k
      = \theta_0 - {\sqrt\tau}\sum\limits_{j=0}^{k-1}\psi(v_j)m(\theta_j)b_j,\\
    v_k
      = v_0 - \sum\limits_{j=0}^{k-1}(v_j-m(\theta_j)^2)\tau
        + \kappa\sqrt\tau\sum\limits_{j=0}^{k-1}m(\theta_j)^2 \omega_j.
  \end{cases}
\end{equation}
Formula \eqref{eq:System2Int} should be thought as an Euler discretization with step-size $\tau$. Indeed,
\[
  \sum_{j=0}^{k-1}\tau(m(\theta_j)^2-v_j)
    \qquad\text{approximates}\qquad
  \int_0^{k\tau}(m(\theta_s)^2-v_s)\,ds.
\]
Instead, since $\omega_k$ and $b_k$ are centred random variable with variance $\tau$, then
\begin{equation}\label{eq:StocIntApprox}
  \sum_{j=0}^{k-1}
    \begin{pmatrix}
      \psi(v_j)m(\theta_j)b_j\\
      \frac1\kappa m(\theta_j)^2 \omega_j  
    \end{pmatrix} 
\end{equation}
is approximating a two-folds stochastic integral with matrix valued integrand.

Notice that the presence of $\sqrt{\tau}$ in front of these terms is necessary to model the fact that as $\tau$ approaches zero, then the dynamic tends to be deterministic.

Now for all $N\in\mathbb{N}$, consider the following linear interpolations
\begin{equation}\label{eq:LinearInter}
  \begin{aligned}
    B^N_t
      &= \sum_{j=0}^{\tN-1}b_j^{(N)}
        + (tN-\tN)b^{(N)}_{\tN},\\
    W^N_t
      &= \sum_{j=0}^{\tN-1}\omega_j^{(N)}
        + (tN-\tN)\omega^{(N)}_{\tN},
  \end{aligned}
\end{equation}
where $(b^{(N)}_j)_{j,N}$ and $(\omega^{(N)}_j)_{j,N}$ are two families of independent random variables, with distribution given by \eqref{eq:2Noises}, with $\tau=1/N$.

Since $(B^N_{k/N},W^N_{k/N})$ is approximating
\[
  \int_0^{k/N}
    \begin{pmatrix}
      1 &0\\
      0 &1\\
    \end{pmatrix}
  \,d
    \begin{pmatrix}
      Z^1_s\\
      Z^2_s
    \end{pmatrix},
\]
then identifying the limit of the law of the sequence of the stochastic processes $(B^N_t,W^N_t)_{t\geq0}$ allows us to find the driving signal of the stochastic integral approximated by \eqref{eq:StocIntApprox}.

The next proposition offers a crucial insight: an additional layer of stochasticity is needed to describe an adaptive stochastic algorithm by a continuous-in-time evolution process. A similar result holds without imposing any constraint on the third moment $|\E[z_k^3]|$, as illustrated in \cref{r:beyond}.

\begin{proposition}\label{Prop}
    The sequence $(Z^N)_{N\geq1} = (B^N,W^N)_{N\geq1}$ converges in law to a two-dimensional Brownian motion.
\end{proposition}
\begin{proof}
  First, we prove the convergence of the finite dimensional laws. Up to an easy generalization, it is enough to show that $(Z^N_s,Z^N_t)$ converges in law to a centred Gaussian random variable with covariance $\begin{pmatrix}s & s\\ s & t\end{pmatrix}$, for all $0\leq s<t\leq 1$.
  
  Consider $0\leq s<t\leq 1$, then there exists $M$ such that $|s-t|>1/M$. For $N>M$ and for any $(\alpha,\beta,\gamma,\delta)\in \R^4$, we have that
  \[
    \begin{aligned}
      \lefteqn{\alpha B^N_t + \beta W^N_t + \gamma B^N_s + \delta W^N_s =}\quad&\\
        &= \sum_{j=0}^{\sN-1}
          \Bigl(\alpha+\gamma)b_j^{(N)}+(\beta+\delta)\omega_j^{(N)}\Bigr)
         + \sum_{j=\sN}^{\tN-1}
           \Bigl(\alpha b_j^{(N)} +\beta \omega_j^{(N)}\Bigr)
           + \alpha(tN-\tN)b_{\tN}^{(N)}\\
        &\quad + \beta(tN-\tN)\omega_{\tN}^{(N)}
               + \gamma(sN-\sN)b_{\sN}^{(N)}
               + \delta(sN-\sN)\omega_{\sN}^{(N)}.
    \end{aligned}
  \]
  Since the last four terms converge trivially to $0$, then we focus on the first two, that we denote respectively by $X^N$ and $Y^N$.

  Let us start with $Y^N$. If
  \[
    \begin{aligned}
      a_N^2
        &:= \E\bigl[\bigl(\alpha \xi_j^{(N)} + \beta \kappa^{-1}((\xi^{(N)}_j)^2-1)\bigr)^2\bigr]
           = \alpha^2 + \beta^2 + 2\alpha\beta\kappa^{-1}\E[(\xi_j^{(N)})^3]\\
      h^{(N)}_j
        &:= \frac{\alpha \xi_j^{(N)}+\beta\kappa^{-1}((\xi^{(N)}_j)^2-1)}
              {a_N\sqrt{\tN -\sN}},
    \end{aligned}
  \]
  then
  \[
    Y^N
      = a_N\sqrt{\tfrac1{N}(\tN -\sN})
        \sum_{j=\sN}^{\tN -1}h^{(N)}_j
  \]
  converges in law to a centred Gaussian random variable of variance $(\alpha^2+\beta^2)(t-s)$, where the convergence is due to Berry-Esseen inequality (\cite{feller_introduction_1991}) and the low skewness condition \eqref{eq:moments}.

  With similar considerations, it follows that $X^N$ converges in law to a centred Gaussian random variable of variance $((\alpha+\gamma)^2+(\beta+\delta)^2)s$. Notice that $(X^N)_N$ and $(Y^N)_N$ are independent, therefore $X^N+Y^N$ converges in law to a centred Gaussian random variable of variance $((\alpha+\gamma)^2+\beta+\delta)^2)s+(\alpha^2+\beta^2)(t-s))$, which implies the required convergence thanks to the Lévy continuity theorem, see \cite{feller_introduction_1991}.

  Since we have shown the convergence of the finite dimensional laws, it is sufficient to prove tightness of the laws of $Z^N$ as measures on $C([0,1],\mathbb{R}^2)$ in order to conclude the proof. This follows from a trivial modification of Theorem 4.20 in \cite{karatzas_brownian_1998}.
\end{proof}

\begin{remark}\label{r:SDE}
  With the above proposition at hands, it is not difficult to prove that if one considers \eqref{eq:System2} without the factor $\sqrt\tau$ in front of the noise terms, namely,
  \begin{equation}\label{eq:System3}
    \begin{cases}
      \theta_{k+1}
        = \theta_k - \psi(v_k)m(\theta_k)b_k,\\
      v_{k+1}
        = v_k - \tau(v_k-m(\theta_k)^2) + \kappa m(\theta_k)^2 \omega_k,
    \end{cases}
  \end{equation}
  then the discrete dynamics converges in law to the solution of the following SDE,
  \begin{equation}\label{eq:sde}
    \begin{cases}
      d\theta_t
        = - \psi(v_t)m(\theta_t)\,dB_t,\\
      dv_t
        = - (v_t - m(\theta_t))^2\,dt
          + \kappa m(\theta_t)^2\,dW_t,
    \end{cases}
  \end{equation}
  where $B$, $W$ are independent standard Brownian motions. More precisely, write \eqref{eq:System3} in integral form,
  \[
    \begin{cases}
      \theta_k
        = \theta_0 - \sum\limits_{j=0}^{k-1}\psi(v_j)m(\theta_j)b_j,\\
      v_k
        = v_0 - \sum\limits_{j=0}^{k-1}(v_j-m(\theta_j)^2)\tau
          + \kappa\sum\limits_{j=0}^{k-1}m(\theta_j)^2\omega_j.
    \end{cases}
  \]
  and define
  \[
    \begin{aligned}
      \theta^N_t
        &= \theta_{\tN}
          - (tN - \tN)\psi(v_\tN)m(\theta_\tN)b_\tN^N,\\
      v^N_t
        &= v_\tN
          - \tfrac1N(tN - \tN)(v_\tN - m(\theta_\tN)^2) + \kappa (tN - \tN)m(\theta_\tN)^2 \omega_\tN^N,
    \end{aligned}
  \]
  then
  \[
    \begin{aligned}
    \theta_t^N
      &= \theta_0 - \int_0^t \Bigl(\sum_{j=0}^{N-1}\psi(v_j)m(\theta_j)\mathbf{1}_{[\frac{j}{N},\frac{j+1}{N})}(s)\Bigr)\,dB_s^N,\\
    v^N_t
      &= v_0 - \int_0^t\Bigl(\sum_{j=0}^{N-1}(v_j-m(\theta_j)^2)\mathbf{1}_{[\frac{j}{N},\frac{j+1}{N})}(s)\Bigr)\,ds
        + \kappa\int_0^t \Bigl(\sum_{j=0}^{N-1}m(\theta_j)^2\mathbf{1}_{[\frac{j}{N},\frac{j+1}{N})}(s)\Bigr)\,dW_s^N,
    \end{aligned}
  \]
  where $B^N$, $W^N$ are defined in \eqref{eq:LinearInter}. The processes $(\theta^N_t,v^N_t)_{t\in[0,1]}$ converge in law to the solutions of \eqref{eq:sde} with initial condition $(\theta_0,\omega_0)$. This result though goes beyond the scopes of the paper and we omit the proof.
\end{remark}

\begin{remark}
  A natural order-1 limiting stochastic dynamics for \eqref{eq:System2}, suggested also by \cref{r:SDE} above, is
  \[
    \begin{cases}
      d\theta_t
        = - \sqrt\tau\psi(v_t)m(\theta_t)\,dB_t,\\
      dv_t
        = - (v_t - m(\theta_t))^2\,dt
          + \kappa\sqrt\tau m(\theta_t)^2\,dW_t,
    \end{cases}
  \]
  with independent Brownian motions $B$, $W$. Using the notion of quadratic variation, as already mentioned in \cref{s:toymodel}, we want to show that order-1 approximations must contain independent Brownian components, and thus our intuition is the correct one. Recall the filtration $\Fc_k=\sigma(\theta_0,v_0,\xi_0,\dots,\xi_{k-1})$, then the increments of co-variation between the two discrete processes $(\theta_k)_k$, $(v_k)_k$ are given by $\E[(\theta_{k+1} - \theta_k)(v_{k+1} - v_k)\mid\Fc_k]$, which turn out to be zero, since $b_k$ is centred and $b_k,\omega_k$ are uncorrelated.
\end{remark}

\subsection{Setting of the problem}\label{as:setting}

We first recall the setting of the problem. Consider a loss function $\hat f = \hat f(x,\theta)$, a labelled dataset $(x_j)_{j=1}^M$, and a family $(\gamma_{k,i})_{k\in \N,1\leq i\leq B}$ of \iid random variables with uniform law over $\{1,\dots,M\}$, where $B$ is batch size. If $f_\gamma(\theta) = \hat f(x_\gamma,\theta)$, the objective function to be optimised is $f(\theta) = \E[f_{\gamma_{1,1}}(\theta)]$. In particular $\nabla f_{\gamma_{k,i}}(\theta)=\nabla_{\theta}\hat f (x_{\gamma_{k,i}},\theta)$, and $\nabla f(\theta)=\E[\nabla f_{\gamma_{1,1}}(\theta)]$. Batch sampling defines an estimator for the gradient of the objective function, that at each step $k$, with step-size $\tau$, of any stochastic optimization algorithm can be decomposed as
\[
  g_k(\theta)
    = \nabla f(\theta)
      + \frac{1}{B}\sum_{i=1}^B(\nabla f_{\gamma_{k,i}}(\theta)-\nabla f(\theta))
    =: \nabla f(\theta)+\frac{z_k}{\sqrt{\tau B}},
\]
where the last formula defines $z_k$. The sequence $(z_k)_k$ is a family of independent centred random variables with covariance $\tau\Sigma(\theta)$, given by
\[
  \Sigma(\theta)
    := \E\bigl[(\nabla f_{\gamma_{k,1}}(\theta)-\nabla f(\theta))(\nabla f_{\gamma_{k,1}}(\theta)-\nabla f(\theta))^T\bigr].
\]
In our work a relevant role will be played by the covariance matrix $M(\theta)$ of $\hat{z_k}=(\tau^{-1/2} z_k)^{\odot2}$, namely
\[
  M(\theta)
    =\E\bigl[(\hat{z}_k-\Sigma_d(\theta))(\hat{z}_k-\Sigma_d(\theta))^T\bigr].   
\]
\begin{remark}
  The present setting describes SGD noise as label noise. We remark that all our result would hold equally in the case of online noise, and in general as long as the random gradient $g_k$ has the above structure, with iid fluctuations around the true gradient.
\end{remark}

We assume preliminary that $\nabla f$ is Lipschitz-continuous, and that $\Sigma^{1/2}$ is bounded and Lipschitz-continuous. Further more restrictive assumptions will be considered below in the statements of our theorems.

In order to manage the degeneracies of the problem, namely the non-differentiability of the square root of the second momentum, see \cref{r:Remark on phi}, or the non-positivity of the order-2 SDEs, see \cref{r:negative}, we introduce a regularization of the preconditioning matrix. Let $\phi:\R\to\R$ be a smooth function such that 
\begin{itemize}
  \item $\phi(x)\geq c_1>0$ for all $x\in\R$;
  \item $\phi(x)=x$ for all $x\geq c_2>c_1$.
\end{itemize}
We will replace the preconditioner $\sqrt{v}+\epsilon$ with $\sqrt{\phi(v)}+\epsilon$ in both update rules. Notice that, as long as the initial condition $v_0$ of the second momentum is (component-wise) positive, if the algorithm is run for $N$ steps, and if $c_2$ is chosen so that $c_2\leq\beta^N\min (v_0)_i$ (replace $\beta$ with $\beta_2$ for Adam), then $v_k\geq c_2$ (component-wise), and thus the stochastic algorithms run unchanged, regardless of the function $\phi$. A heuristic explanation is that we change the problem in scales that are smaller than the scale of the learning rate, so that the discrete dynamics does not see it. In contrast, the SDEs originate from an asymptotic limit in small learning rate, and thus ``see'' the small scales.

With this in mind, we re-formulate the discrete algorithms in the two scaling, and \emph{afterwards} we introduce the regularization. We remark that, as long as the initial value of the second momentum is positive, the order-1 continuous dynamics stay unchanged as well, and the overall effect of the regularization is purely at a technical level (see \cref{r:is_solution}). In contrast, the order-2 dynamics do not stay positive, in general (see \cref{r:negative} and \cref{r:is_solution2} for further explanations)

In the \regimezero regime, with $\tau=\eta$, we use the values \eqref{eq:values_zero} and use the regularization, to obtain,
\begin{equation}\label{eq:RMSprop_reg_zero}
  \begin{cases}
    \theta_{k+1}
      = \theta_{k} - \tau g_k \odot P(u_k),\\
    u_{k+1}
      = u_k - \lambda_0\tau u_k + \lambda_0\tau g_k^{\odot 2},\\
    g_k
      = \nabla f(\theta_k) + \frac\sigma{\sqrt\tau}z_k,
      \end{cases}
\end{equation}
and
\begin{equation}\label{eq:Adam_reg_zero}
  \begin{cases}
    \theta_{k+1}
      = \theta_{k} - \tau\bar\Gamma_k m_{k+1} \odot \bar Q_k(u_k),\\
    m_{k+1}
      = m_k - \lambda_1\tau m_k + \lambda_1\tau g_k,\\
    u_{k+1}
      = u_k - \lambda_2\tau u_k + \lambda_2\tau g_k^{\odot 2},\\
    g_k
      = \nabla f(\theta_k) + \frac\sigma{\sqrt\tau}z_k.
  \end{cases}
\end{equation}
where
\begin{equation}\label{eq:defPbarGamma}
  \begin{aligned}
  P(u)
    &= (\sqrt{\phi(u)} + \epsilon)^{-1}, &
  \bar Q_k(u)
    &= (\sqrt{\phi(u)} + \epsilon\sqrt{\ogamma_2(k)})^{-1},\\
  \ogamma_i(k)
    &= \bigl(1 - (1 - \lambda_i\tau)^k\bigr),\quad i=1,2,\quad &
  \bar\Gamma_k
    &= \frac{\sqrt{\ogamma_2(k)}}{\ogamma_1(k+1)}.
  \end{aligned}
\end{equation}
In the \regimeuno, with $\tau=\eta^2$, we use the values \eqref{eq:values_uno} and use the regularization, to obtain,
\begin{equation}\label{eq:RMSprop_reg_uno}
  \begin{cases}
    \theta_{k+1}
      = \theta_{k} - \tau g_k \odot P(u_k),\\
    u_{k+1}
      = u_k - \lambda_0\tau u_k + \lambda_0\tau^2 g_k^{\odot 2},\\
    g_k
      = \nabla f(\theta_k) + \frac\sigma\tau z_k,
      \end{cases}
\end{equation}
and
\begin{equation}\label{eq:Adam_reg_uno}
  \begin{cases}
    \theta_{k+1}
      = \theta_{k} - \tau\bar\Gamma_k m_{k+1} \odot \bar Q_k(u_k),\\
    m_{k+1}
      = m_k - \lambda_1\tau m_k + \lambda_1\tau g_k,\\
    u_{k+1}
      = u_k - \lambda_2\tau u_k + \lambda_2\tau^2 g_k^{\odot 2},\\
    g_k
      = \nabla f(\theta_k) + \frac\sigma\tau z_k.
  \end{cases}
\end{equation}

\subsection{Scaling rules}\label{as:scaling}

First, we recast the order-1 continuous dynamics, when the regularization is taken into account. In the balistic regime, the order-1 SDEs, corresponding to \eqref{eq:sde1_rmsprop_zero} and \eqref{eq:sde1_adam_zero}, with the regularization included, are, respectively,
\begin{equation}\label{eq:sde1_rmsprop_reg_zero}
  \begin{cases}
    d\theta_t
      = - \nabla f(\theta_t) \odot P(u_t)\,dt,\\
    du_t
      = \lambda_0(\nabla f(\theta_t)^{\odot 2} + \sigma^2\Sigma_d(\theta_t) - u_t)\,dt.
  \end{cases}
\end{equation}
and
\begin{equation}\label{eq:sde1_adam_reg_zero}
  \begin{cases}
    d\theta_t
      = - \Gamma_t m_t\odot Q_t(u_t)\,dt,\\
    dm_t
      = \lambda_1 (\nabla f(\theta_t) - m_t)\,dt,\\
    du_t
      = \lambda_2(\sigma^2\Sigma_d(\theta_t) + \nabla f(\theta_t)^{\odot 2} - u_t)\,dt,
  \end{cases}
\end{equation}
where
\begin{equation}\label{eq:defQGamma}
  Q_t(u)
    = (\sqrt{\phi(u)} + \epsilon\sqrt{\gamma_2(t)})^{-1},\qquad
  \Gamma_t
    = \frac{\sqrt{\gamma_2(t)}}{\gamma_1(t)},
\end{equation}
and $\gamma_1$, $\gamma_2$ are given as in \eqref{eq:gammas}. We recall that the diagonal $\diag(\Sigma(\theta))$ of $\Sigma(\theta)$ is denoted by $\Sigma_d(\theta)$.

In the \regimeuno regime, the order-1 SDEs, corresponding to \eqref{eq:sde1_rmsprop_uno} and \eqref{eq:sde1_adam_uno}, with the regularization included, are, respectively,
\begin{equation}\label{eq:sde1_rmsprop_reg_uno}
  \begin{cases}
    d\theta_t
      = - \bigl(\nabla f(\theta_t)\,dt + \sigma\Sigma(\theta_t)^{\frac12}\,dW_t\bigr)\odot P(u_t),\\
    du_t
      = \lambda_0(\sigma^2\Sigma_d(\theta_t) - u_t)\,dt.
  \end{cases}
\end{equation}
and
\begin{equation}\label{eq:sde1_adam_reg_uno}
  \begin{cases}
    d\theta_t
      = - \Gamma_t m_t\odot Q_t(u_t)\,dt,\\
    dm_t
      = \lambda_1 (\nabla f(\theta_t) - m_t)\,dt
        + \lambda_1\sigma\Sigma(\theta_t)^{\frac12}\,dW_t,\\
    du_t
      = \lambda_2(\sigma^2\Sigma_d(\theta_t) - u_t)\,dt.
  \end{cases}
\end{equation}

We notice that, while the function $\bar\Gamma$ defined in \eqref{eq:defPbarGamma} and appearing in the Adam discrete dynamics \eqref{eq:Adam_reg_zero} and \eqref{eq:Adam_reg_uno} makes perfectly sense for every $k\geq0$ (although divergent in the asymptotic limit $\tau\to0$), the same unfortunately does not hold for its continuous counterpart $\Gamma$ defined in \eqref{eq:defQGamma}, which is singular at $t=0$. We shall discuss this problem in \cref{r:on_gamma}.

To prove that the above SDEs are order-1 weak approximations of their respective discrete dynamics, we rely on the framework introduced in \cite{li_stochastic_2017,li_stochastic_2019}. In short, the method is based on the following ideas (see \cite[Theorem 3]{li_stochastic_2019}, and \cite[Theorem B.2]{malladi_sdes_2022} for the order-1 non-autonomous version),
\begin{enumerate}
  \item existence and uniqueness, and uniform (in time) control of a sufficient number of moments of the solution of the SDE with respect to the initial condition,
  \item  differentiability of the solution of the SDE with respect to the initial condition, with as many derivatives as needed, with uniform control of moments: this ensures differentiability of the transition semigroup,
  \item uniform control of moments, of sufficiently high order, of the discrete dynamics,
  \item smallness of the short-time increments of the discrete dynamics, within some order,
  \item closeness of the short-time increments of the discrete and continuous dynamics, within some order.
\end{enumerate}
The first two of the above items can be obtained in full generality under suitable assumptions of regularity on the coefficients of the SDE. These are classical results. In the framework of the SGD approximation, see \cite[Theorems 18, 20]{li_stochastic_2019}, where general results for non-autonomous SDE are obtained, and that cover our cases as well. So, to prove \cref{t:scaling}, we shall need to prove control of moments within order $\tau^2$. To this end, we reformulate \cref{t:scaling} in the more precise setting outlined above.

\begin{theorem}\label{at:scaling}
  Fix positive constants $\lambda_0$, $\lambda_1$, $\lambda_2$, $\epsilon$, $\sigma$, $T>0$, and $\tau\in(0,1\wedge T)$, and set $N=\lfloor T/\tau\rfloor$. Assume that $\nabla f, \Sigma^{1/2}\in\Gc^4$, that $\nabla f$, $\Sigma^{1/2}$ are Lipschitz-continuous, and that $\Sigma^{1/2}$ is bounded. Assume also that \cref{BMC} hold.
  \smallskip

  \emph{RMSprop case}. Assume moreover that $\tau<\frac1{2\lambda_0}$. Let $(y_k)_{k\geq0}=(\theta_k,u_k)_{k\geq0}$ be solution of \eqref{eq:RMSprop_reg_zero} and $(Y_t)_{t\geq0}=(\theta_t,u_t)_{t\geq0}$ solution of \eqref{eq:sde1_rmsprop_reg_zero}, such that $Y_0=y_0$. Then $(Y_t)_{t\in[0,T]}$ is a order-1 weak approximation of $(y_k)_{0\leq k\leq N}$, namely for every $g\in\Gc^5$,
  \[
    \max_{0\leq k\leq N}\big|\E[g(y_k)] - \E[g(Y_{k\tau})]\big|
      \leq c\tau,
  \]
  where $c>0$ is a number independent of $\tau$.

  If additionally \cref{LSC} holds with $p=2$, then the same conclusion holds if $(y_k)_{k\geq0}$ is solution of \eqref{eq:RMSprop_reg_uno} and $(Y_t)_{t\geq0}$ is solution of \eqref{eq:sde1_rmsprop_reg_uno}.
  \smallskip

  \emph{Adam case}. Assume moreover that $\tau<\frac1{2\lambda_2}$, fix an arbitrary $t_0\in(0,T)$ and set $k_0=\lfloor t_0/\tau\rfloor$. Let $(y_k)_{k\geq k_0}=(\theta_k,m_k,u_k)_{k\geq k_0}$ be solution of \eqref{eq:Adam_reg_zero} and $(Y_t)_{t\geq t_0}=(\theta_t,m_t,u_t)_{t\geq t_0}$ solution of \eqref{eq:sde1_adam_reg_zero} such that $Y_{t_0}=y_{k_0}$. Then $(Y_t)_{t\in[t_0,T]}$ is a order-1 weak approximation of $(y_k)_{k_0\leq k\leq N}$, namely for every $g\in\Gc^5$,
  \[
    \max_{k_0\leq k\leq N}\big|\E[g(y_k)] - \E[g(Y_{k\tau})]\big|
      \leq c\tau,
  \]
  where $c>0$ is a number independent of $\tau$.
  
  If additionally \cref{LSC} holds with $p=2$, then the same conclusion holds if $(y_k)_{k\geq k_0}$ is solution of \eqref{eq:Adam_reg_uno} and $(Y_t)_{t\geq t_0}$ is solution of \eqref{eq:sde1_adam_reg_uno}.
\end{theorem}

\begin{remark}\label{r:is_solution}
  We have already observed that, as long as the initial state of the second momentum is component-wise positive, and if $\tau$ is small enough that $1-\lambda_0\tau\in(0,1)$, which is ensured by \cref{at:scaling} above, then
  \[
    \min_{0\leq k\leq N}\min_i (u_k)_i
      \geq (1 - \lambda_0\tau)^N\min_i (u_0)_i
      >0,
  \]
  so, by suitably choosing the constant $c_2$ appearing in the definition of the regularising function $\phi$, the solutions of \eqref{eq:RMSprop_reg_zero} and \eqref{eq:RMSprop_reg_uno} are also solutions of the same dynamics but with $\phi\equiv1$.

  If we turn to the continuous dynamics \eqref{eq:sde1_rmsprop_reg_zero} and \eqref{eq:sde1_rmsprop_reg_uno}, we see that in both cases, $du_t\geq -u_t$, therefore
  \[
    \min_{t\in[0,T]}\min_i (u_t)_i
      \geq e^{-\lambda_0 T}\min_i (u_0)_i
      >0,
  \]
  therefore, by possibly choosing a smaller value for the constant $c_2$ in the definition of the regularising function $\phi$, the solutions \eqref{eq:sde1_rmsprop_reg_zero}, respectively of \eqref{eq:sde1_rmsprop_reg_uno}, are also solutions of \eqref{eq:sde1_rmsprop_zero}, respectively of \eqref{eq:sde1_rmsprop_uno}. In conclusion, as long as the initial state of the second momentum is component-wise positive, and as long as $\phi$ is calibrated carefully, there is no effect in the introduction of the regularisation and our results deal with the original dynamics. Since calibrations depend on initial values and parameters, it can be actually done before running the algorithms, and the regularisation is only a technical artifact to allow for rigorous proofs. We shall see (\cref{r:is_solution2}) that this is not the case for order-2 continuous dynamics.

  The same comments equally apply to Adam discrete dynamics \eqref{eq:Adam_reg_zero} and \eqref{eq:Adam_reg_uno}, by just replacing $\lambda_0$ with $\lambda_2$
\end{remark}

\begin{remark}\label{r:on_gamma}
  \cref{at:scaling} above, in the Adam version, proves order-1 matching between the discrete and the continuous dynamics. The technical reason behind this requirement is that the function $\Gamma$ (defined in \eqref{eq:defQGamma}), appearing in both continuous order-1 SDEs for Adam, is singular at $t=0$. Likewise, its discrete counterpart $\bar\Gamma$ (defined in \eqref{eq:defPbarGamma}), appearing in both discrete Adam dynamics, is asymptotically singular in the limit $\tau\downarrow0$. We notice that while a possible alternative solution could have been to tame the singularity of $\bar\Gamma$ and $\Gamma$, coherently with our approach for the non-differentiability of the preconditioning matrix, we have kept this approach to be consistent with the findings of \cite{malladi_sdes_2022}.

  We would like to discuss an interpretation of this problem. Such functions originate from the normalization of the first and second momenta of the gradient, so that they are unbiased estimators of the true quantities. In the regime when the (original) parameters $\beta$, $\beta_1$, $\beta_2$ approach $1$, in the first iterations the expected momentum are almost zero, and normalization actually amplifies fluctuations. Evaluating the dynamics at some positive time $t_0>0$ is thus equivalent to considering the problem out of this transient regime. An equivalent effect can be obtained by introducing a lower bound (inducing a macroscopic scale) to the quantities $\Gamma$, $\bar\Gamma$.
\end{remark}

Before reporting the proof of \cref{at:scaling}, let us introduce some notations which will be often used in the rest of the paper: given two functions $f(y)$ and $g(y)$, then 
\[|f(y)|\lesssim\tau ^{\alpha}, f(y)=g(y)+O(\tau^{\alpha})\]
means that there exist $K_1(y),K_2(y)\in \Gc$ s.t.
\[|f(y)|\le K_1(y)\tau^{\alpha}, |f(y)-g(y)|\le K_2(y)\tau^{\alpha}.\]

\begin{proof}[Proof of \cref{at:scaling}]
  Define, as in the statement of the theorem, $y_k=(\theta_k,u_k)$ and $Y_t=(\theta_t,u_t)$ for RMSprop dynamics, and $y_k=(\theta_k,m_k,u_k)$ and $Y_t=(\theta_t,m_t,u_t)$ for Adam dynamics, started at the same initial condition. For $k\geq0$ define $\Delta y_k = y_{k+1}-y_k$, and for $t\geq0$ define $\Delta_\tau Y_t = Y_{t+\tau}-Y_t$.
  by \cref{at:approx}, we need to
  \begin{itemize}
    \item identify the values of $\Delta_k y$ and $\Delta_\tau Y_t$ (or, more precisely, of $\E[\Delta_k y\mid\Fc_k])$, where, as before, $\Fc_k=\sigma(y_0,z_0,\dots,z_{k-1})$, and of $\E[\Delta_\tau Y_t\mid\Fc_t]$, where $(\Fc_t)_{t\geq0}$ is the natural filtration of $(Y_t)_{t\geq0}$),
    \item prove that their second moments are small within order $\tau^2$,
    \item prove that their difference is small within order $\tau^2$.
  \end{itemize}
  We first identify the values of the increments. Recall that, given a SDE
  \[
    dY_t
      = b(t,Y_t)\,dt + \sigma(t,Y_t)\,dW_t,
  \]
  adapted to the filtration $(\Fc_t)_{t\geq0}$, and with $b$, $\sigma\in\Gc^2$, by It\=o's formula,
  \[
    |\E[\Delta_\tau Y_t\mid\Fc_t] - \tau b(t,Y_t)|
      \lesssim\tau^2,
  \]
  and
  \[
    |\E[(\Delta_\tau Y_y)(\Delta_\tau Y_y)^T\mid\Fc_t]|
      \lesssim\tau^2.
  \]
  Therefore, for \eqref{eq:sde1_rmsprop_reg_zero} and \eqref{eq:sde1_adam_reg_zero}, $\E[\Delta_\tau Y_t\mid\Fc_t] = \Delta_\tau Y_t$ is equal to
  \begin{equation}\label{eq:sde1_Delta_zero}
    \tau\begin{pmatrix}
      - \nabla f(\theta_t) \odot P(u_t)\\
      \lambda_0(\nabla f(\theta_t)^{\odot 2} + \sigma^2\Sigma_d(\theta_t) - u_t)
    \end{pmatrix},\qquad
    \tau\begin{pmatrix}
      - \Gamma_t m_t\odot Q_t(u_t)\\
      \lambda_1 (\nabla f(\theta_t) - m_t)\\
      \lambda_2(\nabla f(\theta_t)^{\odot 2} + \sigma^2\Sigma_d(\theta_t) - u_t)
    \end{pmatrix},
  \end{equation}
  while for \eqref{eq:sde1_rmsprop_reg_uno}, \eqref{eq:sde1_adam_reg_uno}, $\E[\Delta_\tau Y_t\mid\Fc_t]$ is equal respectively to
  \[
    \tau\begin{pmatrix}
      - \nabla f(\theta_t)\odot P(u_t)\\
      \lambda_0(\sigma^2\Sigma_d(\theta_t) - u_t)
    \end{pmatrix},\qquad
    \tau\begin{pmatrix}
      - \Gamma_t m_t\odot Q_t(u_t)\\
      \lambda_1 (\nabla f(\theta_t) - m_t)\\
      \lambda_2(\sigma^2\Sigma_d(\theta_t) - u_t)
    \end{pmatrix}.
  \]
  In the \regimezero regime, since $\E[m_{k+1}\mid\Fc_k]=m_k + \lambda_1\tau(\nabla f(\theta_k) - m_k)$ and
  \[
    \E[g_k\mid\Fc_k]
      = \nabla f(\theta_k),\qquad
    \E[g_k^{\odot 2}\mid\Fc_k]
      = \nabla f(\theta_k)^{\odot 2} + \sigma^2\Sigma_d(\theta_k),
  \]
  the increments $\E[\Delta_k y\mid\Fc_k])$ for \eqref{eq:RMSprop_reg_zero} and \eqref{eq:Adam_reg_zero} are, respectively,
    \begin{equation}\label{eq:sgd_Delta_zero}
    \tau\begin{pmatrix}
      - \nabla f(\theta_k) \odot P(u_k)\\
      \lambda_0(\nabla f(\theta_k)^{\odot 2} + \sigma^2\Sigma_d(\theta_k) - u_k)
    \end{pmatrix},\qquad
    \tau\begin{pmatrix}
      - \bar\Gamma_k\bigl(m_k + \lambda_1\tau(\nabla f(\theta_k) - m_k)\bigr) \odot \bar Q_k(u_k)\\
      \lambda_1(\nabla f(\theta_k) - m_k)\\
      \lambda_2(\nabla f(\theta_k)^{\odot 2} + \sigma^2\Sigma_d(\theta_k) - u_k)
    \end{pmatrix}.
  \end{equation}
  Likewise, in the \regimeuno regime,
  \[
    \E[g_k^{\odot 2}\mid\Fc_k]
      = \nabla f(\theta_k)^{\odot 2} + \sigma^2\tau^{-1}\Sigma_d(\theta_k),
  \]
  therefore, the increments $\E[\Delta_k y\mid\Fc_k])$ for \eqref{eq:RMSprop_reg_uno}, \eqref{eq:Adam_reg_uno} are, respectively,
  \[
    \tau\begin{pmatrix}
      \nabla f(\theta_k)\odot P(u_k)\\
      \lambda_0(\tau\nabla f(\theta_k)^{\odot 2} + \sigma^2\Sigma_d(\theta_k) - u_k)
    \end{pmatrix},\qquad
    \tau\begin{pmatrix}
      - \Gamma_k\bigl(m_k + \lambda_1\tau(\nabla f(\theta_k) - m_k)\bigr) \odot \bar Q_k(u_k)\\
      \lambda_1(\nabla f(\theta_k) - m_k)\\
      \lambda_2(\tau\nabla f(\theta_k)^{\odot 2} + \sigma^2\Sigma_d(\theta_k) - u_k)
    \end{pmatrix}.
  \]
  Smallness of the second moments within order $\tau$ follows from regularity of coefficients and \cref{BMC}, as in \cite[Lemma C.9]{malladi_sdes_2022}

  Finally, to conclude, it is sufficient to estimate
  \[
    \E[\Delta_\tau Y_{k\tau}] - \E[\Delta_k y]
  \]
  conditional to $Y_{k\tau}=y_k$, within order $\tau^2$. Let us consider the case of Adam in the \regimezero regime, the other estimates follow with similar arguments. By comparing the second formula in \eqref{eq:sde1_Delta_zero} and \eqref{eq:sgd_Delta_zero}, we see that we need to compare only the first component,
  \[
    \E[\Delta_\tau \theta_{k\tau}] - \E[\Delta_k \theta]
      = \tau\bigl(\Gamma_{k\tau}m\odot Q_{k\tau}(u) - \bar\Gamma_k m\odot\bar Q_k(u)\bigr) - \lambda_1\tau^2\bar\Gamma_k(\nabla f(\theta)-m),
  \]
  since the other components are equal to zero, and we can include in the error the term with pre-factor $\tau^2$. By our choice of $\gamma_i$ (see \eqref{eq:gammas}), we have that $\bar\gamma_i(k) - \gamma_i(k\tau) \sim\tau^2$, therefore $\E[\Delta_\tau \theta_{k\tau}] - \E[\Delta_k \theta]\sim\tau^2$. This computation is elementary in the case of RMSprop, and crucially relies on the fact that $k\geq k_0$ and $\geq t_0$ in the case of Adam, see \cref{r:on_gamma}.
\end{proof}

\subsection{Effective order-2 continuous equations for adaptive algorithms}\label{as:order2}

We turn to order-2 approximations of the discrete dynamics. We start by writing the SDEs modelling the discrete dynamics, taking into account the regularization we have introduced. For simplicity, we shall take $\lambda_0=\lambda_1=\lambda_2=1$ in \eqref{eq:values_zero} and \eqref{eq:values_uno}, as the exact value of the constants does not play a significant role.

We first consider the \regimezero regime. The order-2 SDEs corresponding to \eqref{eq:sde2_rmsprop_zero} and \eqref{eq:sde2_adam_zero} are, respectively,
\begin{equation}\label{eq:sde2_rmsprop_reg_zero}
  \begin{cases}
    \begin{aligned}
      d\theta_t
        &= - P(u_t)\odot \nabla f(\theta_t)\,dt\\
        &\quad -\frac12\tau\Bigl(P(u_t)^{\otimes 2}\nabla^2 f(\theta_t)\nabla f(\theta_t)
           + \frac12P(u_t)^{\otimes 2}\frac{\phi'(u_t)}{\sqrt{\phi(u_t)}}\odot(J(\theta_t)-u_t)\odot \nabla f(\theta_t)\Bigr)\,dt\\
        &\quad + \sqrt{\tau}\bigl(\sigma\diag(P(u_t))\Sigma(\theta_t)^{\frac12}\,dW_t\bigr),
    \end{aligned}\\
    \begin{aligned}
      du_t
        &= (J(\theta_t) - u_t)\,dt\\
        &\quad + \tfrac12\tau\Bigl((J(\theta_t) - u_t)
         + \nabla J(\theta_t)(\nabla f(\theta_t)\odot P(u_t))\Bigr)\\
         &\quad +\sqrt{\tau}\bigl(-2\sigma\diag(\nabla f(\theta_t))\Sigma(\theta_t)^{\frac12}\,dW_t
         + \sigma^2 M(\theta_t)^{\frac12}\,dB_t\bigr),
    \end{aligned}
  \end{cases}
\end{equation}
and
\begin{equation}\label{eq:sde2_adam_reg_zero}
  \begin{cases}
    \begin{aligned}
      d\theta_t
        & = \Gamma_t\Bigl(- m_t\odot Q_t(u_t)
          - \tau \frac14 Q_t(u_t)^{\otimes2}\frac{\phi'(u_t)}{\sqrt{\phi(u_t)}} \odot m_t\odot (J(\theta_t) - u_t)\Bigl)\,dt,\\
        &\quad + \frac12\tau\Bigl(\Gamma_t(m_t - \nabla f(\theta_t))\odot Q_t(u_t)
            + m_t\odot \partial_t(\Gamma_t Q_t(\cdot))(u_t)\Bigr)\,dt\\ 
    \end{aligned}\\
    \begin{aligned}
      dm_t
        &= (\nabla f(\theta_t)-m_t)\,dt
          + \frac12\tau\Bigl((\nabla f(\theta_t)-m_t)
          + \Gamma_t\bigl(\nabla^2f(\theta_t)(m_t\odot Q_t(u_t))\bigr)\Bigr)\,dt\\
        &\quad + \sqrt\tau\sigma\Sigma(\theta_t)^{\frac12}\,dW_t,
    \end{aligned}\\
    \begin{aligned}
      du_t
        & = (J(\theta_t) - u_t)\,dt
          + \frac12\tau\Bigl((J(\theta_t)-u_t)
          + \nabla J(\theta_t)(m_t\odot \Gamma_tQ_t)\Bigr)\,dt\\
        & \quad +\sqrt\tau\Bigl(- 2\sigma\diag(\nabla f(\theta_t))\Sigma(\theta_t)^{\frac12}\,dW_t
          + \sigma^2 M(\theta_t)^{\frac12}\,dB_t\Bigr),
    \end{aligned}
  \end{cases}
\end{equation}
where $M$ is the covariance of the second moment of the noise, defined in \eqref{eq:M}, $P$ and is defined in \eqref{eq:defPbarGamma}, $J(\theta)=\nabla f(\theta)^{\odot2} + \sigma^2\diag\Sigma(\theta)$, $\Gamma$ and $Q_t$ are defined in \eqref{eq:defQGamma}, and $(B_t)_{t\geq0}$, $(W_t)_{t\geq0}$ are two \emph{independent} $d$-dimensional Brownian motions. In the \regimeuno regime, the order-2 SDEs for RMSprop and Adam are, respectively,
\begin{equation}\label{eq:sde2_rmsprop_reg_uno}
  \begin{cases}
    \begin{aligned}
      d\theta_t
        &= - P(u_t)\odot\nabla f(\theta_t)\,dt\\
        &\quad - \frac12\tau\Bigl(P(u_t)^{\otimes 2}\nabla^2 f(\theta_t)\nabla f(\theta_t)
           + \frac12P(u_t)^{\otimes2}\frac{\phi'(u_t)}{\sqrt{\phi(u_t)}}\odot(\sigma^2\Sigma_d(\theta_t)-u_t)\odot\nabla f(\theta_t)\\
        &\quad +2\sigma^2\sum_{h,k=1}^d({P(u_t)}^{\otimes 2}_{h,k}\Sigma_{h,k}(\theta_t)\nabla(\partial^2_{h,k}f(\theta_t))\odot P(u_t)\Bigr)\,dt\\
        &\quad + \sigma\diag(P(u_t))\Sigma(\theta_t)^{\frac12}\,dW_t
           + \tau\Lambda_1(\theta_t,u_t)\,dW_t,
    \end{aligned}\\
    \begin{aligned}
      du_t
        &= (\sigma^2\Sigma_d(\theta_t)-u_t)\,dt
          + \frac12\tau\Bigl((\sigma^2\Sigma_d(\theta_t)-u_t)
          + 2\nabla f(\theta_t)^{\otimes2}\\
        &\quad - \sigma^2\nabla \Sigma_d(\theta_t)(\nabla f(\theta_t)\odot P(u_t)\bigr)
          - \tau\sigma^4 \sum_{h,k=1}^d {P(u_t)}^{\otimes 2}_{h,k}\Sigma_{h,k}\partial^2_{h,k}\Sigma_d(\theta_t)\Bigr)\,dt\\
        &\quad + \tau\Lambda_2(\theta_t,u_t)\,dW_t
          + \sqrt{\tau}M(\theta_t)^{\frac12}\,dB_t.
    \end{aligned}
  \end{cases}
\end{equation}
and
\begin{equation}\label{eq:sde2_adam_reg_uno}
  \begin{cases}
    \begin{aligned}
      d\theta_t
        &= - \Gamma_t m_t\odot Q_t(u_t)\,dt
           + \frac12\tau\Bigl(\Gamma_t(m_t - \nabla f(\theta_t))\odot Q_t(u_t)
           + m_t\odot (\partial_t(\Gamma_t Q_t(\cdot))(u_t))\\
        &\quad -\frac12 \Gamma_t Q_t(u_t)^{\otimes2}\frac{\phi'(u_t)}{\sqrt{\phi(u_t)}}\odot(\sigma^2\Sigma_d(\theta_t)-u_t)\Bigr)\,dt\\
        &\quad + \frac12\tau\sigma\Gamma_t\diag(Q_t(u_t))\Sigma(\theta_t)^{\frac12}\,dW_t,
    \end{aligned}\\
    \begin{aligned}
      dm_t
        &=  (\nabla f(\theta_t)-m_t)\,dt
          + \frac12\tau\Bigl((\nabla f(\theta_t)-m_t)
          + \Gamma_t\nabla^2f(\theta_t)(m_t\odot Q_t(u_t))\Bigr)\,dt\\
        &\quad + \sigma\Sigma(\theta_t)^{\frac12}\,dW_t
          + \frac12\tau\sigma \Gamma_t\Bigl(\sum_{h=1}^d (m_t\odot Q_t(u_t))_h\partial_h\Sigma(\theta_t)^{\frac12}\Bigr)\,dW_t,
    \end{aligned}\\
    \begin{aligned}
      du_t
        &=  (\sigma^2\Sigma_d(\theta_t)-u_t)\,dt
          + \frac12\tau\Bigl(\nabla f(\theta_t)^{\odot 2} +(\sigma^2\Sigma_d(\theta_t)-u_t)\\
          &\quad+ \sigma^2\Gamma_t\nabla\Sigma_d(\theta_t)(m_t\odot Q_t(u_t))\Bigr)\,dt\\
        &\quad -2\tau\sigma\diag(\nabla f(\theta_t))\Sigma(\theta_t)^{\frac12}\,dW_t
         + \sqrt{\tau}M(\theta_t)^{\frac12}\,dB_t.
    \end{aligned}
  \end{cases}
\end{equation}
The terms $\Lambda_1$, $\Lambda_2$ in formula \eqref{eq:sde2_rmsprop_reg_uno} are as follows,
\begin{equation}
  \begin{aligned}
    \Lambda_1(\theta_t,u_t)
      &= \frac12\sigma\nabla_\theta\bigl(\nabla f(\theta_t)\odot P(u_t)\bigr)\diag(P(u_t))\Sigma(\theta_t)^{\frac12}\\
      &\quad - \frac14\diag\Bigl(P(u_t)^{\otimes 2}\frac{\phi'(u_t)}{2\sqrt{\phi(u_t)}}\odot(\sigma^2\Sigma_d(\theta_t)-u_t)\Bigr)\Sigma(\theta_t)^{\frac12}\\
      &\quad -\frac12\sigma\sum_{h=1}^d P(u_t)_h\partial_hf(\theta_t)\diag(P(u_t))\partial_h \Sigma^{\frac12}(\theta_t)\\
      &\quad + \sigma^3\sum_{h,k=1}^d P(u_t)^{\otimes 2}_{h,k}\Sigma_{h,k}\diag(P(u_t))\partial^2_{h,k} \Sigma(\theta_t)^{\frac12}\\
      &\quad + \sigma^3\Bigl(\diag(P(u_t))\Sigma(\theta_t)^{\frac12}\Bigr)^{-1} \Bigl(\sum_{h,k=1}^d P(u_t)^{\otimes 2}_{h,k}\Sigma_{h,k} \partial_h \Sigma(\theta_t)^{\frac12}\partial_k\Sigma^{\frac12}(\theta_t)\Bigr),
    \end{aligned}
\end{equation}
and
\begin{equation}
  \Lambda_2(\theta_t,u_t)
    = - 2\sigma\diag(\nabla f(\theta_t))\Sigma(\theta_t)^{\frac12}
      - \frac12\sigma^3\nabla\Sigma_d(\theta_t)\diag(P(u_t))\Sigma(\theta_t)^{\frac12}.
\end{equation}
As already noticed in \cref{r:limitInvert}, in general, $\Lambda_1$ is well-defined only if $\Sigma(\theta_t)^{\frac12}$ is invertible (unless $\Sigma$ is diagonal).

\begin{theorem}\label{at:scaling2}
  Fix positive constants $\lambda_0$, $\lambda_1$, $\lambda_2$, $\epsilon$, $\sigma$, $T>0$, and $\tau\in(0,1\wedge T)$, and set $N=\lfloor T/\tau\rfloor$. Assume that $\nabla f$, $\Sigma^{1/2}\in\Gc^6$, that $\nabla f$, $\nabla^2 f$, $\Sigma^{1/2}$, $\nabla\Sigma$ are bounded and Lipschitz-continuous, and that $M^{1/2}$ is Lipschitz-continuous. Assume also that \cref{BMC} hold.
  \smallskip

  \emph{RMSprop case}. Assume moreover that $\tau<\frac1{2\lambda_0}$ and that \cref{LSC} holds with $p=\frac52$. Let $(y_k)_{k\geq0}=(\theta_k,u_k)_{k\geq0}$ be solution of \eqref{eq:RMSprop_reg_zero} and $(Y_t)_{t\geq0}=(\theta_t,u_t)_{t\geq0}$ solution of \eqref{eq:sde2_rmsprop_reg_zero}, such that $Y_0=y_0$. Then $(Y_t)_{t\in[0,T]}$ is a order-2 weak approximation of $(y_k)_{0\leq k\leq N}$, namely for every $g\in\Gc^5$,
  \[
    \max_{0\leq k\leq N}\big|\E[g(y_k)] - \E[g(Y_{k\tau})]\big|
      \leq c\tau^2,
  \]
  where $c>0$ is a number independent of $\tau$.
  
  If additionally \cref{LSC} holds with $p=3$ and $\Sigma(\theta)^{1/2}$ is uniformly invertible (that is, the smaller eigenvalue has a positive lower bound independent of $\theta$), with Lipschitz-continuous inverse, then the same conclusion holds if $(y_k)_{k\geq0}$ is solution of \eqref{eq:RMSprop_reg_uno} and $(Y_t)_{t\geq0}$ is solution of \eqref{eq:sde2_rmsprop_reg_uno}.
  \smallskip

  \emph{Adam case}. Assume moreover that $\tau<\frac1{2\lambda_2}$ and that \cref{LSC} holds with $p=\frac52$, fix an arbitrary $t_0\in(0,T)$ and set $k_0=\lfloor t_0/\tau\rfloor$. Let $(y_k)_{k\geq k_0}=(\theta_k,m_k,u_k)_{k\geq k_0}$ be solution of \eqref{eq:Adam_reg_zero} and $(Y_t)_{t\geq t_0}=(\theta_t,m_t,u_t)_{t\geq t_0}$ solution of \eqref{eq:sde2_adam_reg_zero} such that $Y_{t_0}=y_{k_0}$. Then $(Y_t)_{t\in[t_0,T]}$ is a order-2 weak approximation of $(y_k)_{k_0\leq k\leq N}$, namely for every $g\in\Gc^6$,
  \[
    \max_{k_0\leq k\leq N}\big|\E[g(y_k)] - \E[g(Y_{k\tau})]\big|
      \leq c\tau^2,
  \]
  where $c>0$ is a number independent of $\tau$.
  
  If additionally \cref{LSC} with $p=3$, then the same conclusion holds if $(y_k)_{k\geq k_0}$ is solution of \eqref{eq:Adam_reg_uno} and $(Y_t)_{t\geq t_0}$ is solution of \eqref{eq:sde2_adam_reg_uno}.
\end{theorem}

We remark that in Adam the requirement that the approximation should hold only from  a positive time is due to the singularity of the normalizations, see \cref{r:on_gamma}.

\begin{remark}\label{r:is_solution2}
  In contrast with order-1 approximations, for the order-2 approximating SDEs \eqref{eq:sde2_rmsprop_reg_zero}, \eqref{eq:sde2_rmsprop_reg_uno}, \eqref{eq:sde2_adam_reg_zero}, \eqref{eq:sde2_adam_reg_uno} the choice of the regularising function $\phi$ \emph{does} matter, and the solutions of these approximating SDEs do not solve the corresponding unregularised (that is, $\phi(x)=x$) versions. As already mentioned in \cref{r:negative}, this is not a contradiction, and the explanation is built-in in the result: our SDEs are order-2 approximations, and noise in the equations for the second momentum is ``small'', in the sense that there is a small pre-factor (a power of $\tau$). The approximations resolve the details of the discrete dynamics to some confidence, and the uncertainty of their driving noise is within the level of confidence.
\end{remark}

The proof of \cref{at:scaling2} will require showing that the difference between first and second moments of $\Delta y_k$ and $\Delta_{\tau}Y_{\tau k}$ is small (of order $\tau^3)$. Therefore, we report a lemma estimating the first and second moments of $\Delta_{\tau}Y_t$.

\begin{lemma}\label{lemma:momEst}
    Given $\tau >0$, $d,m\in \N$ and a $2d$-dimensional Brownian motion $(W_t)_t$. Let $Z_t=(X_t,Y_t)\in \R^{d+m}$ be the unique probabilistic strong solution to the following SDE:
    \begin{equation}\label{eq:genericSDE}
        \begin{cases}
            dZ_t=(b_0+\tau b_1)(t,Z_t)dt+\sigma_0(Z_t)+(\sqrt{\tau}\sigma_1+\tau \sigma_2)(t,Z_t)dW_t,\\
            Z_{t_0}=z,
        \end{cases}
    \end{equation}
    where $b_0,b_1,\sigma_0,\sigma_1,\sigma_2 \in \Gc^4$ are Lipschitz and with linear growth in space, uniformly in time.\\
    Moreover, assume that $\sigma_0\sigma_1^T=\sigma_2\sigma_1^T\equiv0$, that $b_0,b_1,\sigma_0,\sigma_1,\sigma_2$ are in $C^2_t\Gc^4$, i.e $C^2$ in time with derivatives in $\Gc^4$, uniformly in time.
    
    Then for all $i,j\in\{1,\dots,d+m\}$,
    \begin{align}
        &\begin{aligned}\label{e:1stMom}
            \E[(Z_{t_0+\tau}-z)_i]&=\tau (b_0)_i(t_0,z)\\
            &+\tau^2\left((b_1)_i+\frac{b_0\cdot \nabla(b_0)_i}2+\frac{Tr[\nabla^2(b_0)_i\sigma_0\sigma_0^T]}4+\frac{\partial_tb_0}2\right)(t_0,z)+O(\tau^3)
        \end{aligned}\\
        &\begin{aligned}\label{e:2ndMom}
           \E[(Z_{t_0+\tau}-z)_i(Z_{t_0+\tau}-z)_j]&=\tau (\sigma_0\sigma_0^T)_{i,j}(z)+\tau^2\Big\{(b_0\otimes b_0)_{i,j}\\
           &+\frac{Tr[(e_j\overset{S}{\otimes}\nabla(b_0)_i+e_i\overset{S}{\otimes}\nabla(b_0)_j)\sigma_0\sigma_0^T+\nabla^2(\sigma_0\sigma_0^T)_{i,j}\sigma_0\sigma_0^T]}{4}\\
           &+(\sigma_1\sigma_1^T)_{i,j}+\frac{\nabla (\sigma_0\sigma_0^T)_{i,j}\cdot b_0}2+(\sigma_0\sigma_2^T+\sigma_2\sigma_0^T)_{i,j}\Big\}(t_0,z)\\
           &+O(\tau^3),\\
        \end{aligned}
    \end{align}
    where $e_k$ is the $k$-th element of the canonical basis in $\R^{d+m}$ and $v\overset{S}{\otimes}w=v\otimes w+w\otimes v$.
\end{lemma}

\begin{proof}
    Given a $C^1_tC^2_x$ function $f$ such that $f(t,\cdot)\in \Gc^1$ uniformly in time. Then,  It\=o's formula and uniform estimates on the moments of $Z_s$ (see [\cite{li_stochastic_2019},Theorem 19]) imply that
    \begin{align}\label{e:Ito}
        \E[f(t,Z_t)]=f(t_0,z)+\E\left[\int_{t_0}^t\left(\partial_tf+b\cdot\nabla f+\frac{Tr[(\nabla^2f)\sigma\sigma^T]}{2}\right)(s,Z_s)\,ds\right],
    \end{align}
    where $\sigma:=\sigma_0+\sqrt{\tau}\sigma_1+\tau\sigma_2$ and $b:=b_0+\tau b_1$.\\
    Then for $f(h)=(h-z)_i$, we have that
    \begin{equation}
        \begin{aligned}
            \E[(Z_{t_0+\tau}&-z)_i]=\E\left[\int_{t_0}^{t_0+\tau}b_i(s,Z_s)\,ds\right]=\tau \,b_i(t_0,z)\\
            &+\E\left[\int_{t_0}^{t_0+\tau}\int_{t_0}^s\left(\partial_tb_0+b_0\cdot\nabla (b_0)_i+\frac{Tr[\nabla^2(b_0)_i\sigma_0\sigma_0^T]}{2}\right)(r,Z_r)\,dr\,ds\right]+O(\tau^3).
        \end{aligned}
    \end{equation}
    Then, \cref{e:1stMom} follows from applying \cref{e:Ito} to \[\left(\partial_tb_0+b_0\cdot\nabla (b_0)_i+\frac{Tr[\nabla^2(b_0)_i\sigma_0\sigma_0^T]}{2}\right)(r,h).\]

Instead, proving \cref{e:2ndMom} requires applying \cref{e:Ito} to
$f_1(h)=(h-z)_i(h-z)_j$. Indeed,
\begin{equation}\label{e:2ndMomComp}
    \begin{aligned}
        \E[(Z_{t_0+\tau}&-z)_i(Z_{t_0+\tau}-z)_j]=\E\left[\int_{t_0}^{t_0+\tau}\left((Z_s-z)_ib_j+(Z_s-z)_jb_i\right)(s,Z_s)\,ds\right]\\
        &+\E\left[\int_{t_0}^{t_0+\tau}(\sigma_0\sigma_0^T+\tau\,\sigma_0\sigma_2^T+\tau\,\sigma_2\sigma_0^T+\tau\,\sigma_1\sigma_1^T)_{i,j}(s,Z_s)\,ds\right]+O(\tau^3).
    \end{aligned}
\end{equation}
Now, applying \cref{e:Ito} to functions $f(s,h)=(h-z)_kg(s,h)$ allows to reformulate the first summand as follows:
\begin{equation}
 \begin{aligned}
    \E&\int_{t_0}^{t_0+\tau}\int_{t_0}^s(Z_r-z)_j(b_0\cdot\nabla(b_0)_i+\partial_tb_0)(r,Z_r)+(Z_r-z)_i(b_0\cdot\nabla(b_0)_j+\partial_tb_0)(r,Z_r) \\&+2(b_0)_i(b_0)_j(r,Z_r)+\frac12Tr\left[\left(e_j\overset{S}\otimes \nabla(b_0)_i+e_i\overset{S}\otimes \nabla(b_0)_j\right)\sigma_0\sigma_0^T\right](r,Z_r)\,dr+O(\tau^3)\\
    &=\E\int_{t_0}^{t_0+\tau}\int_{t_0}^s\left(2(b_0)_i(b_0)_j+\frac12Tr\left[\left(e_j\overset{S}\otimes \nabla(b_0)_i+e_i\overset{S}\otimes \nabla(b_0)_j\right)\sigma_0\sigma_0^T\right]\right)(r,Z_r)\,dr\\&+O(\tau^3)
\end{aligned}   
\end{equation}
Instead, the second term at the RHS of \cref{e:2ndMomComp} is equal to
\begin{equation}
\begin{aligned}
    \tau \,(\sigma_0\sigma_0^T&+\tau\,\sigma_0\sigma_2^T+\tau\,\sigma_2\sigma_0^T+\tau\,\sigma_1\sigma_1^T)_{i,j}(t_0,z)\\
    &+\E\int_{t_0}^{t_0+\tau}\int_{t_0}^s\left(b_0\cdot\nabla(\sigma_0\sigma_0^T)_{i,j}+\frac12Tr[\nabla^2(\sigma_0\sigma_0^T)_{i,j} \sigma_0\sigma_0^T]\right)(r,Z_r)\,dr+O(\tau^3)
\end{aligned}
\end{equation}
Then, the claim follows from applying \cref{e:Ito} to
\begin{align*}
    &b_0\cdot\nabla(\sigma_0\sigma_0^T)_{i,j}+\frac12Tr[\nabla^2(\sigma_0\sigma_0^T)_{i,j} \sigma_0\sigma_0^T],\\
    &(b_0)_i(b_0)_j+\frac12Tr\left[\left(e_j\overset{S}\otimes \nabla(b_0)_i+e_i\overset{S}\otimes \nabla(b_0)_j\right)\sigma_0\sigma_0^T\right].
\end{align*}
\end{proof}

\begin{proof}[Proof of \cref{at:scaling2}]
  We follow the strategy outlined in \cref{as:scaling} and introduced by \cite{li_stochastic_2017,li_stochastic_2019}. As in the proof of \cref{at:scaling}, existence, uniqueness and uniform control of moments of the SDEs and of the derivatives with respect to the initial conditions follow from the general theorems contained in \cite[Appendix A]{li_stochastic_2019}.

  Smallness of the moments of the short-time increments $\cDelta_k(y)$ within order $\tau^3$ follows from regularity of coefficients and \cref{BMC}, as in \cite[Lemma C.9]{malladi_sdes_2022}. One can see that under \cref{a:approx_zero} one-dimensional moments have order $\tau$ because of the small (in $\tau$) coefficient in front of the noise terms. In contrast, under \cref{a:approx_uno}, some noise coefficients are of order 1, thus one-dimensional moments have order $\sqrt\tau$.

  With a control of the moments of the short-time increments, uniform control of moments of the discrete dynamics follows by \cite[Lemma B.5]{malladi_sdes_2022} (that can be applied with either set of assumptions).

Therefore, the proof will be completed once we show that 
\begin{align}
    &|\E[\Delta_{\tau}Y_{k\tau}]-\E[\Delta y_k]|\lesssim\tau^3;\\
    &\label{2ndCond}|\E[(\Delta_{\tau}Y_{k\tau})^{\otimes2}]-\E[(\Delta y_k)^{\otimes 2}]|\lesssim\tau^3;
\end{align}
for all $k\geq 0$ in the RMSprop cases (where $Y=(\theta,u)$) and all $k\geq k_0$ in the Adam ones where $Y=(\theta,m,u)$).

For the sake of readability, we show that the previous inequalities hold for Adam in the balistic regime. 
By standard computations, it follows that
\begin{equation}
\begin{aligned}
\E[\Delta_k\theta]=\bar\Gamma_k(-\tau m_k+&\tau^2(m_k-\nabla f(\theta_k)))\odot \bar Q_k(u_k), \quad \E[\Delta_k m]=\tau (\nabla f(\theta_k)-m_k), \\
&\E[\Delta_ku]=\tau(\nabla f(\theta_k)^{\odot 2}+\sigma^2\Sigma_d(\theta_k)-u_k),
\end{aligned}
\end{equation}
up to $O(\tau^3)$ terms, i.e. additive terms given by $K(\theta_k,m_k,u_k)\tau^3$ with $K \in \Gc$.

Then, \cref{e:1stMom} and \cref{eq:sde2_adam_reg_zero} imply that 
\begin{equation*}
\begin{aligned}
    |\E[\Delta_{\tau}Y_{k\tau}]-\E[\Delta y_k]|&\le\tau\Bigg|\begin{pmatrix}
        -\Gamma_{\tau k} m_k\odot Q_{k\tau}(u_k)\\
        \nabla f(  \theta_k)-m_k\\
        \nabla f(\theta_k)^{\odot 2}+\sigma^2 \Sigma_d(\theta_k)-u_k
    \end{pmatrix}-
    \begin{pmatrix}
        -\bar\Gamma_k m_k\odot \bar Q_k(u_k)\\
        \nabla f(  \theta_k)-m_k\\
        \nabla f(\theta_k)^{\odot 2}+\sigma^2 \Sigma_d(\theta_k)-u_k
    \end{pmatrix}\Bigg|\\
    &+\tau^2\Bigg|(\bar\Gamma_k[m_k-\nabla f(\theta_k)]\odot \bar Q_k(u_k), 0,0)\\
    &-\left(b_1+\frac{b_0\cdot \nabla(b_0)}2+\frac{\partial_tb_0}2\right)(k\tau,\theta_k,m_k,u_k)\Bigg|+O(\tau^3),
\end{aligned}
\end{equation*}
where we recall from \cref{eq:sde2_adam_reg_zero} that 
\[b_0(t,\theta,m,u)=\left( -\Gamma_t m\odot Q_t(u)\right),\nabla f(\theta)-m,J(\theta)-u_t).\]
Notice that $b_1$, in \cref{eq:sde2_adam_reg_zero}, is given by
\[-\left(\frac{b_0\cdot \nabla(b_0)}2+\frac{\partial_tb_0}2\right) +(\Gamma_t[m-\nabla f(\theta)]\odot Q_t(u),0,0).\]
Therefore, our choice of the $\gamma_i$ (see \cref{eq:gammas}) implies that the distance between discrete and continuous first moments is of order $\tau^3$.

Now, standard computations and the definition of $g_k$ (see \cref{as:setting}) imply that
\begin{equation}\label{eq:2ndMomDiscr}
    \begin{aligned}
        &\E[(\Delta_k\theta)^{\otimes 2}]=\bar \tau^2 \Gamma_k^2(m_k\odot \bar Q_k(u_k))^{\otimes 2},\, \E[(\Delta_km)^{\otimes 2}]=\tau^2\left((m_k-\nabla f(\theta_k))^{\otimes 2}+\sigma^2 \Sigma(\theta_k)\right),\\
        &\E[(\Delta_k u)^{\otimes 2}]=\tau^2\left[(u-\nabla J(\theta_k))^{\otimes 2}+\sigma^4 M(\theta_k)+4\sigma^2 \diag(\nabla f(\theta_k))\Sigma(\theta_k)\diag(\nabla f(\theta_k))\right],\\
        &\E[(\Delta_k \theta)\otimes (\Delta_k m)]=\tau^2 \bar \Gamma_k (m_k\odot \bar Q_k(u_k))\otimes(m_k-\nabla f(\theta_k)),\\
        &\E[(\Delta_k \theta)\otimes (\Delta_k u)]=\tau^2 \bar \Gamma_k (m_k\odot \bar Q_k(u_k))\otimes(u_k-J(\theta_k)),\\
        &\E[(\Delta_k m)\otimes (\Delta_k u)]=\tau^2[(m_k-\nabla f(\theta_k))\otimes (u_k-J(\theta_k))+2\sigma^2 \Sigma(\theta_k) \diag (\nabla f(\theta_k))],
    \end{aligned}
\end{equation}
up to additive $O(\tau^3)$ terms.\\
Notice that there is no term of order $\tau$ in the previous equalities. Since the general second moment of an SDE of the form \ref{eq:genericSDE} is given by \cref{e:2ndMom}, then $\sigma_0=\sigma_2=0$ in \cref{eq:sde2_adam_reg_zero}. 
At last, the proof is completed, i.e. \cref{2ndCond} holds too, since $\sigma_1$ in \cref{eq:sde2_adam_reg_zero} has been chosen to cancel out both the second discrete moments, given by \cref{eq:2ndMomDiscr}, and all the spurious terms due to $b_0$ appearing in \cref{e:2ndMom}.

\end{proof}

\subsubsection{A non-autonomous approximation theorem}

In this section, for the sake of completeness, we extend to arbitrary order the order-1 approximation theorem in \cite[Theorem B.2]{malladi_sdes_2022}, together with a version of \cite[Theorem 3]{li_stochastic_2019} for non-autonomous problems.

We first recall the setting introduced in \cite{li_stochastic_2019}.

\begin{definition}
  Denote by $\Gc$ the set of real valued continuous functions defined on $\R^d$ of at most polynomial growth, namely for $f\in\Gc$, there are $a>0$, $c>0$ such that for all $x\in\R^d$,
  \[
    |f(x)|
      \leq c(1+|x|^a).
  \]
  For an integer $m\geq1$, let $\Gc^m$ be the set of functions belonging to $\Gc$ together with their partial derivatives up to order $m$.

  The definition extends to vector valued functions coordinate-wise. Moreover, if a function $g$ depends on additional parameters, by $g\in\Gc$ we mean that polynomial growth holds uniformly in the additional parameters (that is, the number $a$, $c$ in the inequality above do not depend on the additional parameters).
\end{definition}

Given a family of measurable functions $h_k:\R^d\times\Gamma\times[0,\infty)\to\R^d$, and a random variable $\gamma$ with values in $\Gamma$, define the \emph{stochastic iteration},
\begin{equation}\label{eq:iteration}
  y_{k+1}
    = y_k + \tau h_k(y_k,\gamma_k,\tau),
\end{equation}
where $\tau>0$, $(\gamma_k)_{k\geq0}$ is a collection of \iid random variables with common law equal to the law of $\gamma$, and $y_0$ is given and independent of $(\gamma_k)_{k\geq0}$. We shall denote by $(y_k(k_0,y))_{k\geq k_0}$ the solution of the above iteration \eqref{eq:iteration} with initial state $y_{k_0}(k_0,y)=y$.

Consider the SDE,
\begin{equation}\label{eq:sde_tau}
  \begin{cases}
    dY_t
      = b(t,\tau,Y_t)\,dt
        + \sigma(t,\tau,Y_t)\,dW_t,\\
    Y_0
      = y_0,
  \end{cases}
\end{equation}
where $(W_t)_{t\geq0}$ is a standard $m$-dimensional Brownian motion for some $m\geq1$, and $b:[0,T]\times[0,1]\times[0,1]\times\R^d\to\R^d$, $\sigma:[0,T]\times[0,1]\times[0,1]\times\R^d\to\R^{d\times m}$ are continuous functions such that
\begin{itemize}
  \item $b$, $\sigma$ have at most linear growth, uniformly in time and $\tau$,
  \item $b$, $\sigma$ are Lipschitz-continuous, uniformly in time and $\tau$.
\end{itemize}
Under these assumptions, \eqref{eq:sde_tau} has a unique solution for every initial condition. Denote by $(Y_t(s,y))_{t\geq s}$ the solution of \eqref{eq:sde_tau} above with initial condition $Y_s(s,y)=y$.

Define
\[
  \dDelta_k(y)
    = y_{k+1}(k,y) - y
    = \tau h_k(y,\gamma_k,\tau),\qquad
  \cDelta_\tau(y,t)
    = Y_{t+\tau}(t,y) - y,
\]
and notice that, by independence, $\dF_k = \sigma(y_0,\gamma_1,\dots,\gamma_{k-1})$, then
\[
  \E[y_{k+1} - y_k\mid\dF_k]
    = \dDelta_k(y_k).
\]
Likewise, by the Markov property, if $(\cF_t)_{t\geq s}$ is the filtration generated by $(Y_t(s,y))_{t\geq s}$,
\[
  \E[Y_{t+\tau}(s,y) - Y_t(s,y)\mid\cF_t]
    = \cDelta_\tau(Y_t(s,y),t).
\]
For brevity, we shall set $\cDelta_k(y) = \cDelta_\tau(y,k\tau)$.

Due to the intrinsic difference between the set of equations in the two regimes we have identified, we shall state two sets of assumptions. The first covers the \regimezero case and originates from \cite{li_stochastic_2019}.

\begin{assumption}\label{a:approx_zero}
  Given $\alpha\geq1$, assume that
  \begin{enumerate}
    \item the functions $b$, $\sigma\in\Gc^{\alpha+1}$ are Lipschitz (uniformly in time),
    \item there exists $K_1\in\Gc$, independent of $\tau$, such that for every multi-index $\beta\in\N^d$, with $|\beta|\leq\alpha$, every $y\in\R^d$ and every $k=0,1,\dots,N$,
      \[
        \big|\E\bigl[\bigl(\dDelta_k(y)\bigr)^\beta\bigr]
            - \E\bigl[\bigl(\cDelta_k(y)\bigr)^\beta\bigr]\big|
          \leq K_1(y)\tau^{\alpha+1},
      \]
    \item there exist $a>1$ and $K_2\in\Gc$, independent of $\tau$, such that for every multi-index $\beta$ with $|\beta|=\alpha+1$, every $y\in\R^d$ and every $k=0,1,\dots,N$,
      \[
        \big|\E\bigl[|(\dDelta_k(y))^\beta|^a\bigr]\big|
          \leq K_2(y)\tau^{a(\alpha+1)},
      \]
    \item for all $m\geq1$, there is $K_3\in\Gc$, independent of $\tau$ such that for all $y$,
      \[
        \sup_{k=0,1,\dots,N} \E[|y_k(0,y)|^{2m}]
          \leq K_3(y).
      \]
  \end{enumerate}
\end{assumption}
In the above definition by $y^\beta$ for $y\in\R^d$ and a multi-index $\beta\in\N^d$ we mean $y^\beta=\prod_{j=1}^d y_j^{\beta_j}$, and $|\beta|=\beta_1+\beta_2+\dots+\beta_d$.

\begin{remark}
  As in \cite{li_stochastic_2019}, one can relax the requirements on coefficients by turning from derivatives to weak derivatives. In that case, Assumption (2) of \cref{at:approx} is replaced by the following assumption: there are $\rho:(0,1)\to[0,\infty)$ and $K\in\Gc_1$, independent of $\tau$, such that for every multi-index $\beta$, with $|\beta|\leq\alpha$, and every $y\in\R^d$,
  \[
    \big|\E\bigl[\bigl(\dDelta_k(y)\bigr)^\beta\bigr]
        - \E\bigl[\bigl(\cDelta_k(y)\bigr)^\beta\bigr]\big|
      \leq K_1(y)\bigl(\tau\rho(\epsilon) + \tau^{\alpha +1}\bigr).
  \]
\end{remark}

We then give a set of assumptions suitable for the \regimeuno regime, following \cite{malladi_sdes_2022}.

\begin{assumption}\label{a:approx_uno}
  Given $\alpha\geq1$, assume that
  \begin{enumerate}
    \item the functions $b$, $\sigma\in\Gc^{2\alpha+2}$ are Lipschitz (uniformly in time),
    \item there exists $K_1\in\Gc$, independent of $\tau$, such that for every multi-index $\beta\in\N^d$, with $|\beta|\leq2\alpha+1$, every $y\in\R^d$ and every $k=0,1,\dots,N$,
      \[
        \big|\E\bigl[\bigl(\dDelta_k(y)\bigr)^\beta\bigr]
            - \E\bigl[\bigl(\cDelta_k(y)\bigr)^\beta\bigr]\big|
          \leq K_1(y)\tau^{\alpha+1},
      \]
    \item there is a set $P\subset\{1,2,\dots,d\}$ for which the following properties hold,
  \begin{itemize}
    \item there exist constants $c_1 >0$ and $\omega_1>0$, independent of $\tau$, such that for all $k=0,1,\dots,N$ and $y\in\R^d$,
      \[
        \begin{aligned}
          \|\E[(\dDelta_k(y))_P]\|
            &\leq c_1\tau (1 + \|(y)_P\|),\\
          \|\E[(\dDelta_k(y))_R]\|
            &\leq c_1\tau (1+ \|(y)_P\|^{\omega_1})(1+\|(y)_R\|),
        \end{aligned}
      \]
    \item for all $m\geq 1$, there are constants $c_{2m}>0,$ and $\omega_{2m}>0$, independent of $\tau$, such that for all $k=0,1,\dots,N$ and $y\in\R^d$,
      \[
        \begin{aligned}
          \E[\|(\cDelta_k(y))_P\|^{2m}]
            &\leq c_{2m} \tau^m (1+\|(y)_P\|^{2m}),\\
          \E[\|\cDelta_k(y)\|^{2m}_R]
            &\leq c_{2m} \tau^m (1+ \|(y)_P\|_P^{\omega_{2m}})(1+\|(y)_R\|^{2m}).
        \end{aligned}
      \]
  \end{itemize}
  Here $R = \{1,2,\dots,d\}\setminus P$, and $((x_i)_{i=1,2,\dots,d})_P=(x_i)_{i\in P}$.
  \end{enumerate}
\end{assumption}

\begin{remark}
    Condition (3) in \cref{a:approx_uno} above is tailored for adaptive methods, and takes into account the special structure of the problem.
\end{remark}

We finally state the approximation theorem, which is a minor adaptation of \cite[Theorem 3]{li_stochastic_2019} and \cite[Theorem B.2]{malladi_sdes_2022}.

\begin{theorem}\label{at:approx}
  Let $\alpha\geq 1$, $T>0$, $\tau\in(0,1\wedge T)$, and set $N=\lfloor T/\tau\rfloor$. Assume either \cref{a:approx_zero} or \cref{a:approx_uno}.
  Then for every $g\in\Gc^{\alpha+1}$ (in the case of \cref{a:approx_zero}) or $g\in\Gc^{2\alpha+2}$ (in the case of \cref{a:approx_uno}), there is $c>0$ such that
  \[
    \max_{0\leq k\leq N}|\E[g(y_k)] - \E[g(Y_{k\tau})]|
      \leq c\tau^\alpha,
  \]
  where $y_k = y_k(0,y)$ and $Y_t = Y_t(0,y)$, with $y\in\R^d$.
\end{theorem}
\begin{proof}
  Let $u(y,s,t)=\E[g(Y_t(s,y))]$ and $y_k = y_k(0,y_0)$, then we have that
  \[
    \begin{aligned}
      \lefteqn{|\E[g(y_k)]-\E[g(X_{k\tau}(0,y_0))]| =}\qquad&\\
        &= |\E[g(Y_{k\tau}(k\tau,y_k))] - \E[g(Y_{k\tau}(0,y_0))]|\\
        &\leq \sum_{j=0}^{k-1} |\E[g(Y_{k\tau}((j+1)\tau,y_{j+1}))]
          - \E[g(Y_{k\tau}(j\tau, y_j))]|\\
        &= \sum_{j=0}^{k-1} |\E[g(Y_{k\tau}((j+1)\tau,y_{j+1}))]
          - \E[g(Y_{k\tau}((j+1)\tau,Y_{(j+1)\tau}(j\tau,y_j))]|\\
        &= \sum_{j=0}^{k-1} |\E[g(Y_{k\tau}((j+1)\tau),y_j+\dDelta_{j}(y_j))]
          - \E[g(Y_{k\tau}((j+1)\tau, y_j + \cDelta_j(y_j)))]|\\  
        &= \sum_{j=0}^{k-1} |\E\bigl[\E[g(Y_{k\tau}((j+1)\tau),y_j+\dDelta_{j}(y_j))\mid\dF_j]\bigr]\\
        &\qquad  - \E\bigl[\E[g(Y_{k\tau}((j+1)\tau, y_j + \cDelta_j(y_j)))\mid\dF_j]\bigr]|\\
        &= \sum_{j=0}^{k-1} |\E[u_{j+1}(y_j+\dDelta_{j}(y_j))]
          - \E[u_{j+1}(y_j + \cDelta_j(y_j))]|,
    \end{aligned}
  \]
  where for brevity we have set $u_j(y)=u(y,j\tau,k\tau)$.
  
  Under \cref{a:approx_zero}, by \cite[Proposition 25]{li_stochastic_2019}, $u\in \Gc^{\alpha+1}$ uniformly in $s$, $t$, so $u_1, \dots, u_N \in \Gc^{\alpha+1}$ uniformly. By \cref{l:taylor}, there is $K\in\Gc$ such that
  \[
    |\E[g(y_k)]-\E[g(X_{k\tau}(0,y_0))]|
      \leq \tau^{\alpha+1}\sum_{j=0}^{k-1} \E[K(y_j)]
      \lesssim k\tau^{\alpha+1}
      \lesssim \tau^\alpha,
  \]  
  where we have also used condition (4) of \cref{a:approx_zero}. The same conclusion follows under \cref{a:approx_uno}, since \cite[Proposition 25]{li_stochastic_2019} ensures that $u_1, \dots, u_N \in \Gc^{2\alpha+2}$, while \cite[Lemma B.5]{malladi_sdes_2022} ensures that condition (4) of \cref{a:approx_zero} also holds in this case.
\end{proof}

We conclude with a lemma that summarises \cite[Lemma 27]{li_stochastic_2019} and extends \cite[Lemma B.6]{malladi_sdes_2022} to higher order.
\begin{lemma}\label{l:taylor}
  Assume either \cref{a:approx_zero} (set $\gamma=\alpha$) or \cref{a:approx_uno} (set $\gamma=2\alpha+1$). Let $u_1,u_2,\dots,u_J\in\Gc^{\gamma+1}$ uniformly, then there is $K\in\Gc$, independent of $\tau$, such that
  \[
    |\E[u_j(y+\cDelta_k(y))] - \E[u_j(y+\dDelta_k(y))]|
      \leq K(y)\tau^{\alpha+1},
  \]
  for every $j=1,2,\dots,J$, $k=0,1,\dots,N$, $y\in\R^d$.
\end{lemma}
\begin{proof}
  Fix $j\in\{1,2,\dots,J\}$. Use Taylor's formula around $y$ to obtain
  \[
    \begin{aligned}
      \lefteqn{\E[u_j(y+\cDelta_k(y))] - \E[u_j(y+\dDelta_k(y))] =}\qquad&\\
        &= \sum_{\beta\leq\gamma} D^\beta u_j(y)(\E[(\cDelta_k(y))^\beta] - \E[(\dDelta_k(y))^\beta])
        + \E[\cR_j] + \E[\dR_j],
    \end{aligned}
  \]
  and
  \[
    \dR_j
      = \sum_{|\beta|=\gamma+1}\frac{|\beta|}{\beta!}(\dDelta_k(y))^\beta
        \int_0^1(1-\lambda)^{|\beta|-1}D^\beta u_j(y+\lambda\dDelta_k(y))\,d\lambda,
  \]
  and $\cR_j$ is defined similarly but with $\dDelta_k(y)$ replaced by $\cDelta_k(y)$. By either assumption,
  \[
    \sum_{\beta\leq\gamma} D^\beta u_j(y)(\E[(\cDelta_k(y))^\beta] - \E[(\dDelta_k(y))^\beta])
      \leq K(y)\tau^{\alpha+1},
  \]
  for some $K\in\Gc$. We turn to the estimate of the remainders. Consider first $\dR_j$,
  \[
    \begin{aligned}
    |\E\dR_j|
      &\leq c_\beta\sum_{|\beta|=\gamma+1}\int_0^1(1-\lambda)^{|\beta|-1}
        |\E\bigl[(\dDelta_k(y))^\beta D^\beta u_j(y+\lambda\dDelta_k(y))\bigr]|\,d\lambda,\\
      &\leq c_\beta\sum_{|\beta|=\gamma+1}|\E\bigl[|(\dDelta_k(y))^\beta|^a\bigr]^{\frac1a}
        \int_0^1(1-\lambda)^{|\beta|-1}
        \E\bigl[|D^\beta u_j(y+\lambda\dDelta_k(y))|^{a'}\bigr]^{\frac1{a'}}|\,d\lambda,\\
      &\leq K(y)\sum_{|\beta|=\gamma+1}|\E\bigl[|(\dDelta_k(y))^\beta|^a\bigr]^{\frac1a},
    \end{aligned}
  \]
  for a suitable function $K\in\Gc$, where $a>0$, $\frac1{a}+\frac1{a'}=1$. Under \cref{a:approx_zero}, we can use condition (3) therein (taking for $a$ the value provided there) to conclude. Under \cref{a:approx_uno}, we use condition (3) therein, take $a=2$, to obtain again $|\E[\cR_j]|\leq K(y)\tau^{\frac12(\gamma+1)}=K(y)\tau^{\alpha+1}$. The estimate of $\E[\dR]$ proceeds similarly, using estimates of the moments of the increments as in \cite[Lemma 26]{li_stochastic_2019}, that cannbe easily adapted to higher moments.
\end{proof}

\subsection{Experimental results}\label{as:experiments} 
We experimentally evaluate our results by minimizing a function using discrete algorithms and its continuous approximations.  For the RMSProp algorithm, we set $\lambda_0 = 0.5$, while for Adam we use $(\lambda_1, \lambda_2) = (1, 0.5)$. Further details on the experimental setup can be found in Appendix~\ref{as:experiments setup}.\\

\subsection{\regimezero regime}\label{subsec:Experiments_zero}
\cref{fig: reg_zero_RMSProp_1approx_tau_fixed} compare the weak error \[\E[f(X_{k\tau})-f(x_k)]\] along the dynamics for $X_t$ given by the order-1 and by the order-2 approximations of RMSProp in \regimezero regime at $\tau$ fixed to $2^{-3}$. \cref{fig: reg_zero_Adam_1approx_tau_fixed} present the same comparison for Adam.\\
 In \cref{fig: reg_zero_RMSProp_1approx_tau_fixed,fig: reg_zero_Adam_1approx_tau_fixed}, we observe that the first-order deterministic dynamics struggle to provide accurate approximations when the behavior is largely noise-driven as in the final part of the simulation, whereas the second-order approximations maintain a low error throughout the entire dynamics.\\
 \cref{tb:due} shows the weak error the approximation of Adam. As predicted
by our analysis, the order-2 approximation should give a slope = 2
decrease in error as $\tau$ decrease.

\begin{figure}[ht]
    \centering
    \includegraphics[width=.45\linewidth]{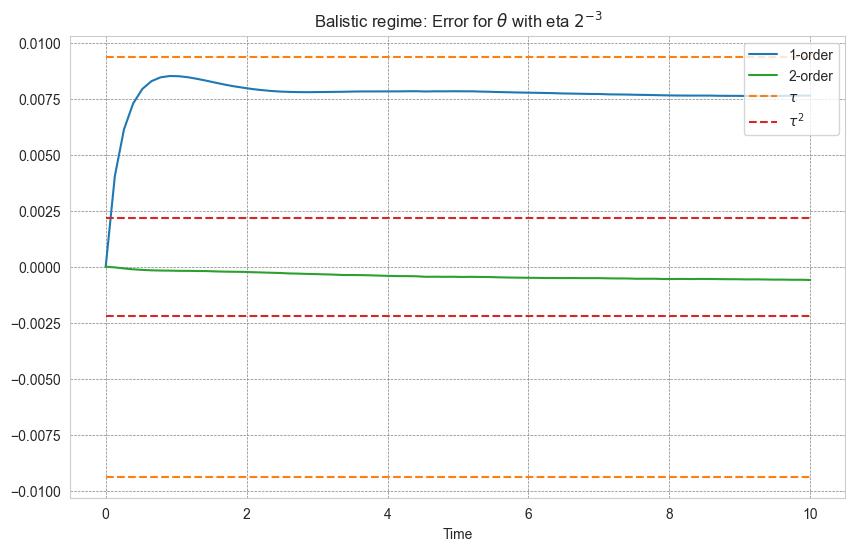}
    \hspace{1 em}
    \includegraphics[width=.45\linewidth ]{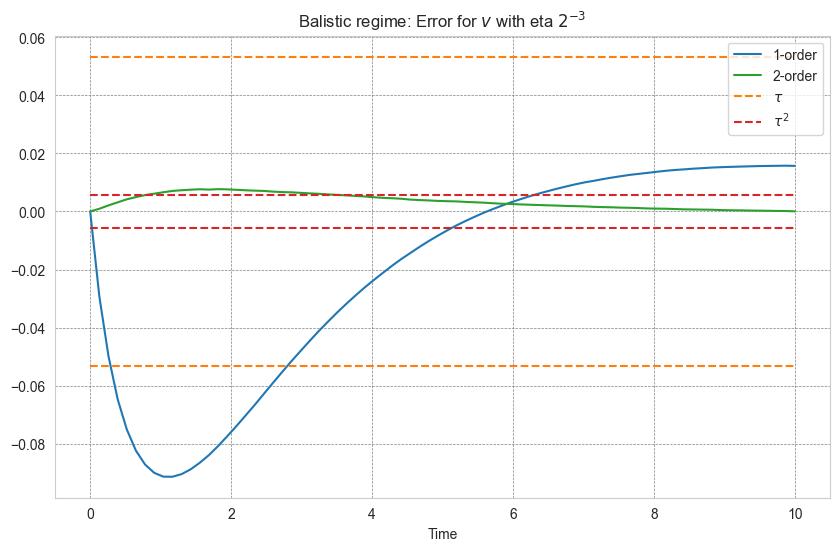}
    \caption{\textit{RMSprop} \regimezero regime: Comparison between the first-order~\eqref{eq:sde1_rmsprop_reg_zero} and second-order~\eqref{eq:sde2_rmsprop_reg_zero} continuous approximations of RMSProp for $\tau = 2^{-3}$.}
    \label{fig: reg_zero_RMSProp_1approx_tau_fixed}
\end{figure}

\begin{table}
  \centering
  \begin{tabular}{ccc}
    \includegraphics[width=.3\linewidth]{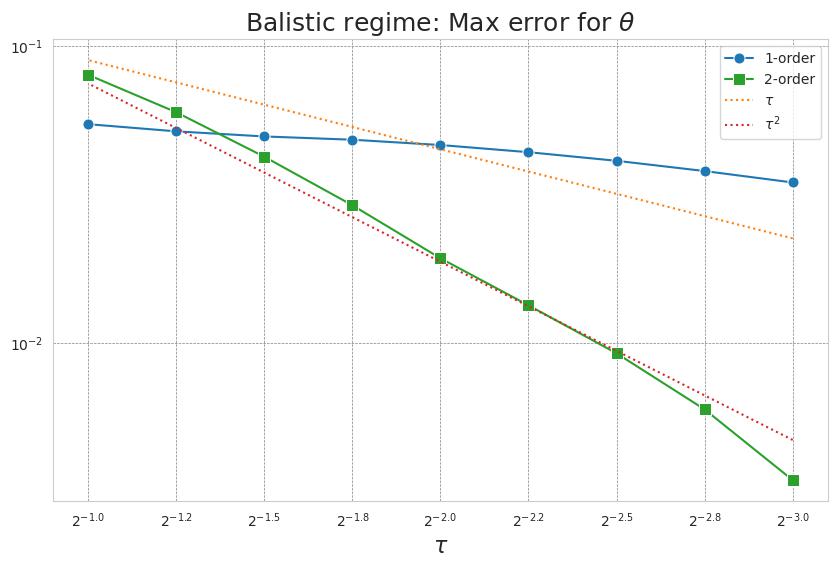} &
    \includegraphics[width=.3\linewidth]{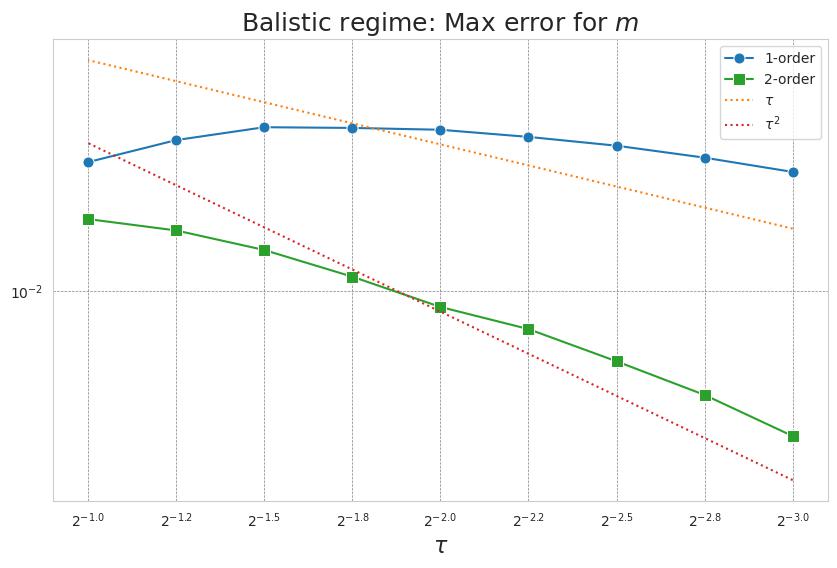}  &

    \includegraphics[width=.3\linewidth]{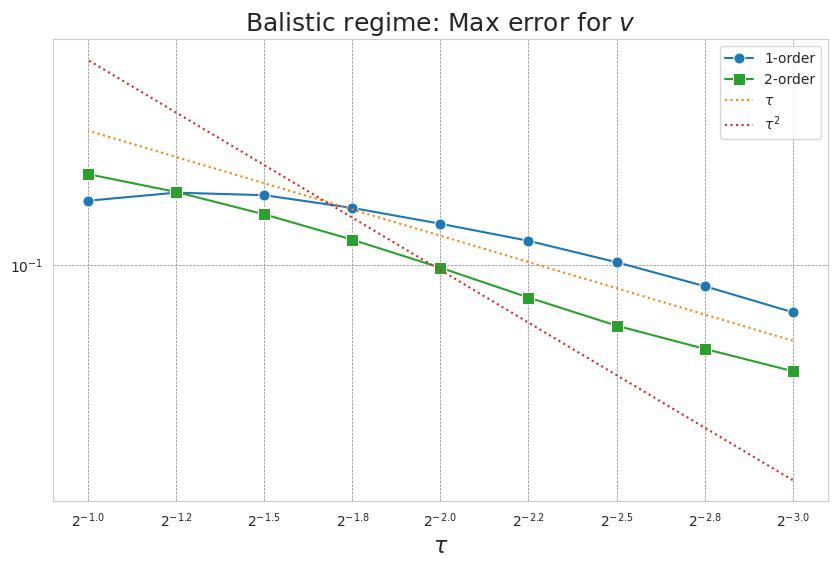}
    
  \end{tabular}
  \caption{\textit{Adam} \regimezero regime: We show the weak error, see \cref{defn:approx}, with test functions $f_1(\theta)=\frac{\|\theta\|_2^2}{d}$, $f_2(m) = \frac{\|m\|_2^2}{d}$ and $f_3(v)=\frac{\|v\|_2^2}{d},$
  where $d$ denotes the dimension of $\theta$, in our case $d=6$.}\label{tb:due}
\end{table}

\begin{figure}[ht]
    \centering
    \includegraphics[width=0.32\textwidth]{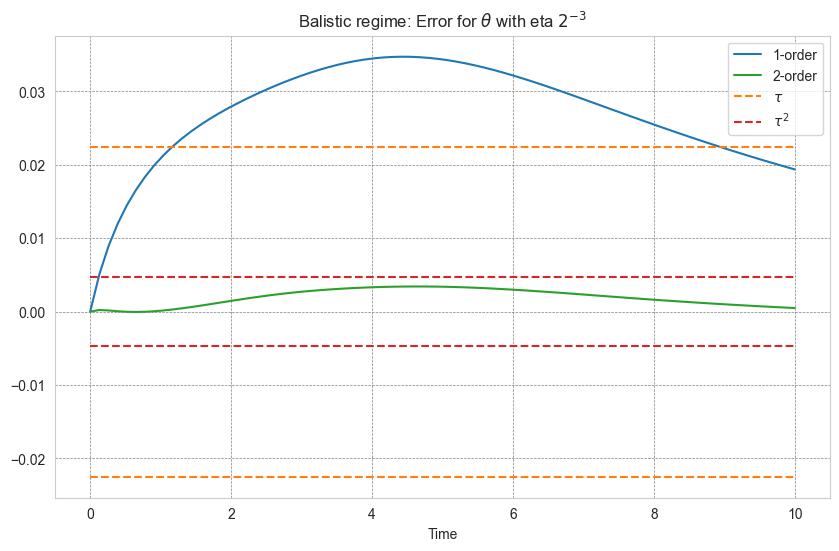}
    \includegraphics[width=0.32\textwidth]{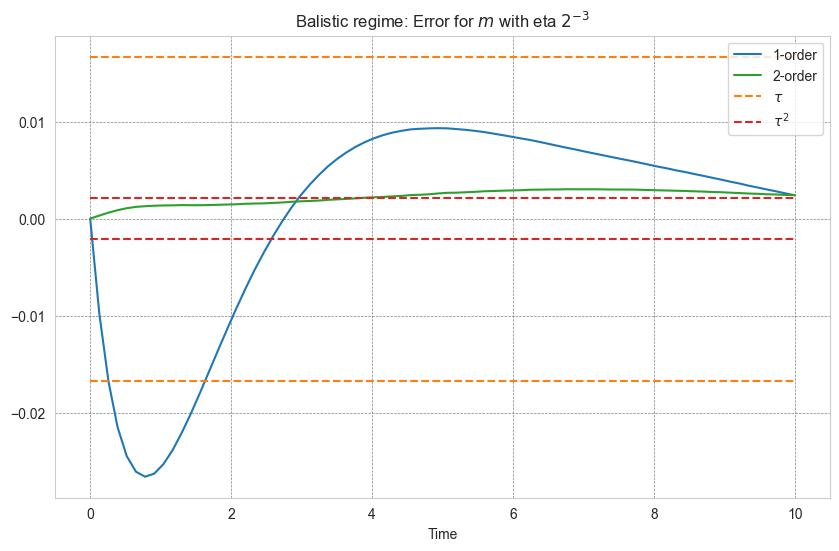}    
    \includegraphics[width=0.32\textwidth]{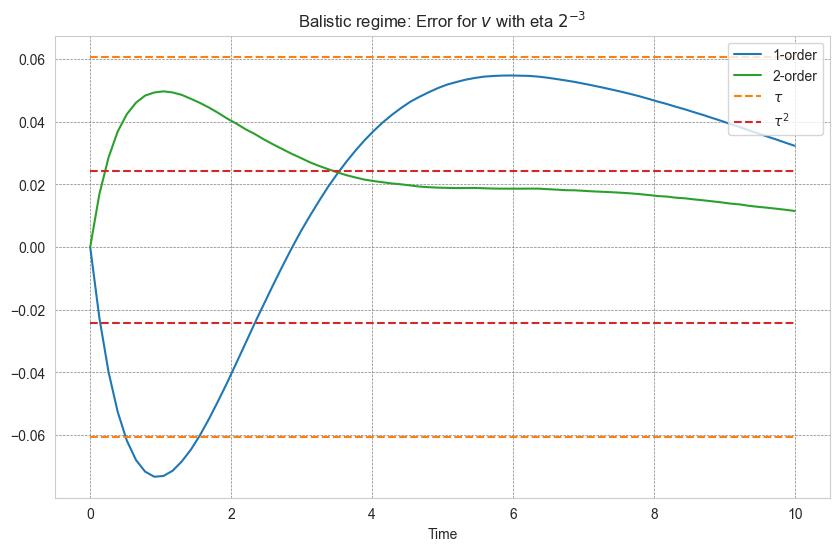}
    
    \caption{\textit{Adam} \regimezero regime: Comparison between the first-order~\eqref{eq:sde1_adam_reg_zero} and second-order~\eqref{eq:sde2_adam_reg_zero} continuous approximations of Adam for $\tau = 2^{-3}$. }
    \label{fig: reg_zero_Adam_1approx_tau_fixed}
\end{figure}

\subsection{\regimeuno regime}
Analogously to \cref{subsec:Experiments_zero}, \cref{fig: reg_uno_RMSProp_1sde_tau_fixed}
compare the weak error along the dynamics of the order-1 and the order-2 approximations of RMSProp in \regimeuno regime, when $\tau$ is fixed to $2^{-3}$. \cref{fig: reg_uno_Adam_1sde_tau_fixed} present the same comparison for Adam.\\
 In particular,  \cref{fig: reg_uno_RMSProp_1sde_tau_fixed,fig: reg_uno_Adam_1sde_tau_fixed} show that at the beginning of the dynamics rapid changes take place and the absence of the second-order term in the drift of first-order approximation leads to a significant error compared to the second-order approximation.\\
\cref{tb:tre} shows the maximum weak error for the approximation of Adam when varying $\tau$. 

\begin{figure}[ht]
    \centering
    \includegraphics[width=.35\linewidth]{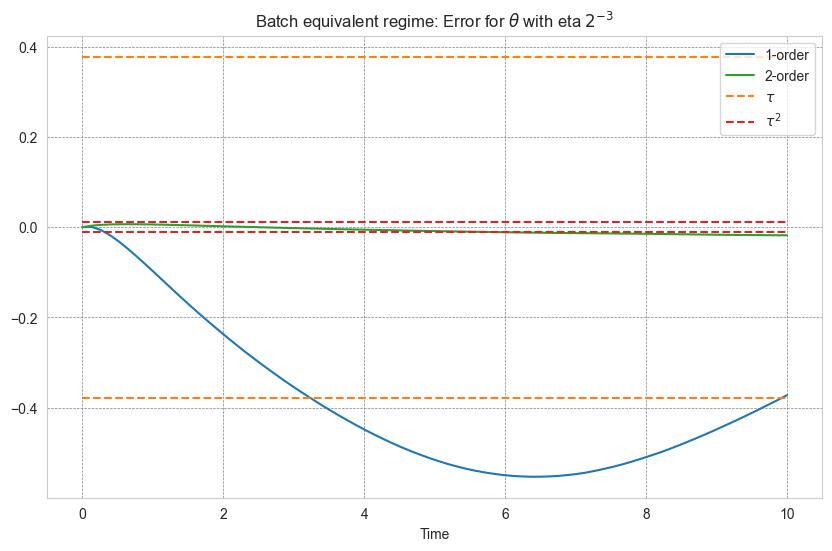}
    \hspace{1 em}
    \includegraphics[width=.35\linewidth ]{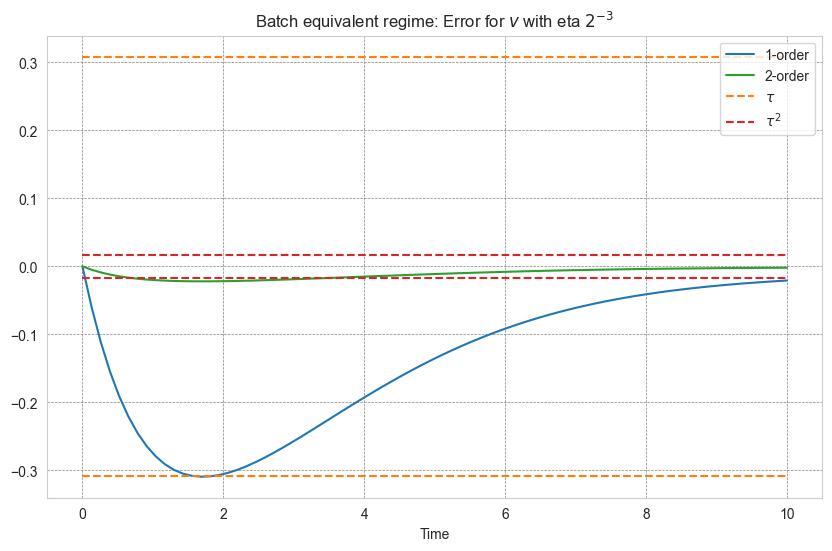}
    \caption{\textit{RMSprop} \regimeuno regime: Comparison between the first-order~\eqref{eq:sde1_rmsprop_reg_uno} and second-order~\eqref{eq:sde2_rmsprop_reg_uno} continuous approximations of RMSProp for $\tau = 2^{-3}$.}
    \label{fig: reg_uno_RMSProp_1sde_tau_fixed}
\end{figure}

\begin{table}
  \centering
  \begin{tabular}{ccc}
    \includegraphics[width=.3\linewidth]{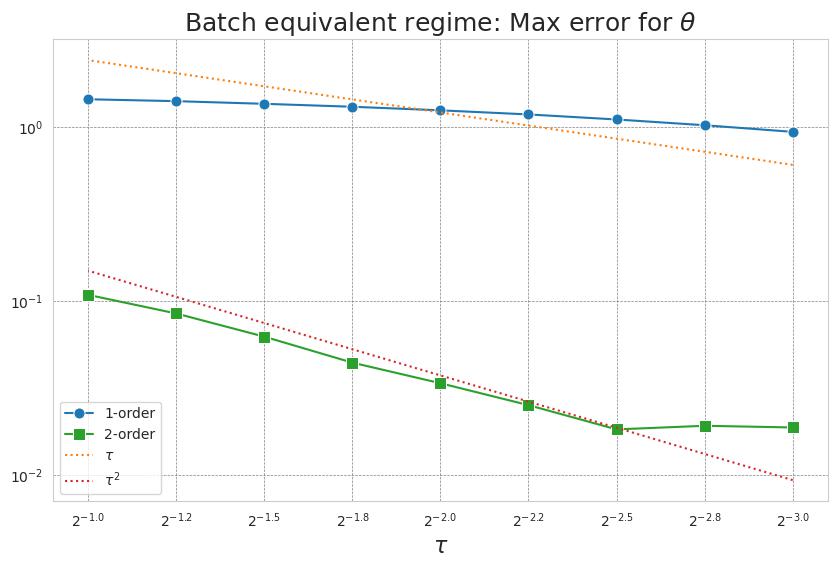} &
    \includegraphics[width=.3\linewidth]{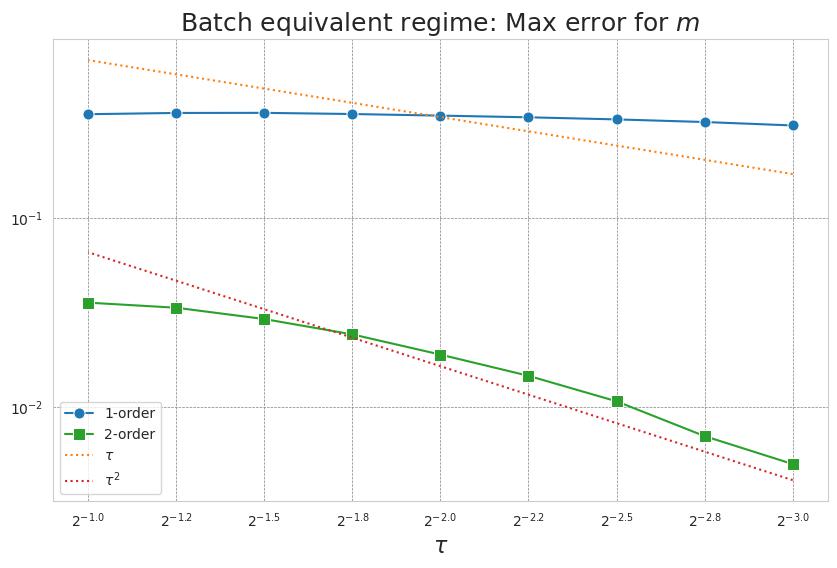}  &

    \includegraphics[width=.3\linewidth]{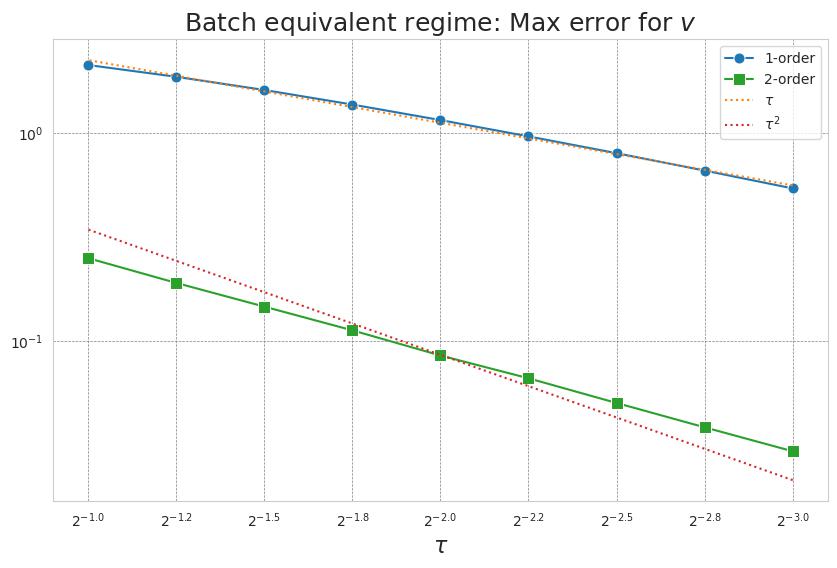}
    
  \end{tabular}
  \caption{\textit{Adam} \regimeuno regime: We show the weak error, see \cref{defn:approx}, with test functions $f_1(\theta)=\frac{\|\theta\|_2^2}{d}$, $f_2(m) = \frac{\|m\|_2^2}{d}$ and $f_3(v)=\frac{\|v\|_2^2}{d},$
  where $d$ denotes the dimension of $\theta$, in our case $d=6$..}\label{tb:tre}
\end{table}

\begin{figure}[ht]
    \centering
    \includegraphics[width=0.32\textwidth]{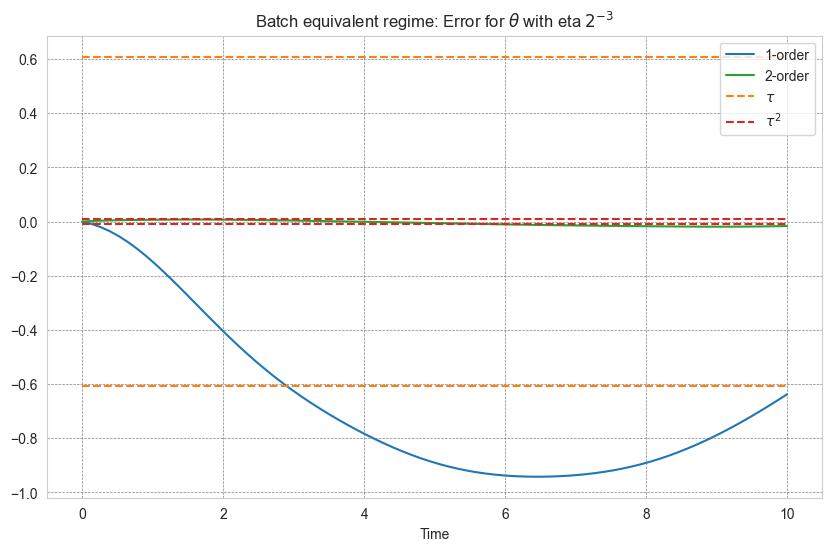}
    \includegraphics[width=0.32\textwidth]{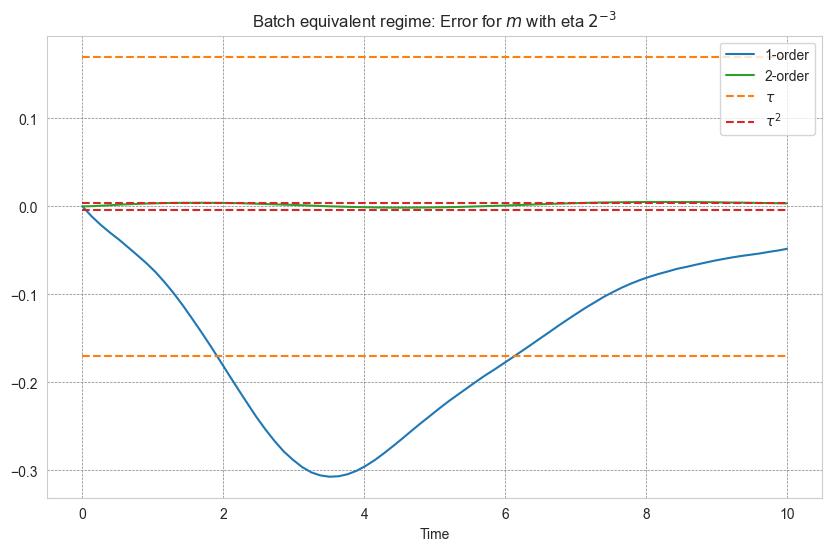}    
    \includegraphics[width=0.32\textwidth]{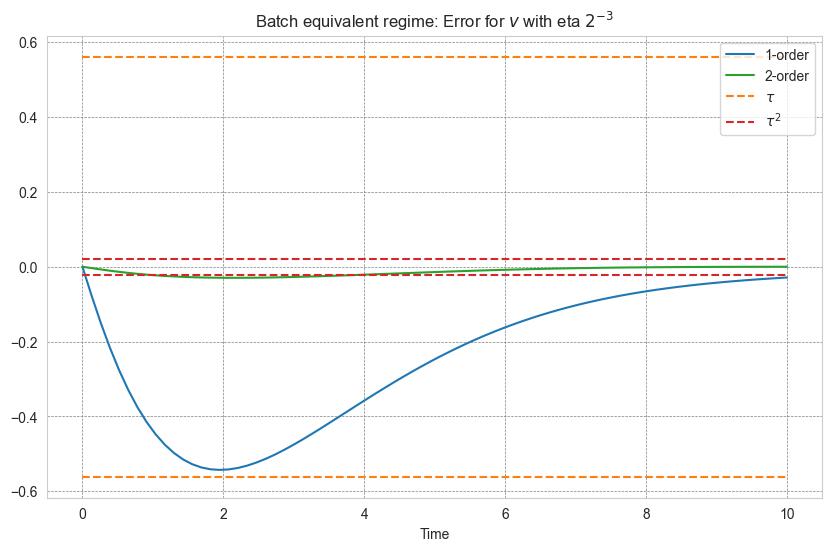}
    
    \caption{\textit{Adam} \regimeuno regime: Comparison between the first-order~\eqref{eq:sde1_adam_reg_uno} and second-order~\eqref{eq:sde2_adam_reg_uno} continuous approximations of RMSProp for $\tau = 2^{-3}$. }
    \label{fig: reg_uno_Adam_1sde_tau_fixed}
\end{figure}

\subsection{Experimental setup}
\label{as:experiments setup}
In our simulations, we minimize the following quadratic objective function:
\begin{align*}
    f_\gamma (\theta) = \frac{1}{2}(\theta-\gamma)^T H
    (\theta-\gamma)- \frac{1}{2}\operatorname{Tr}(H),
\end{align*}
where $\gamma \sim \mathcal{N}(0, Id)$.\\
For the discrete algorithms, $\gamma$ is sampled from a fixed synthetic Gaussian dataset consisting of 128,000 samples, which is shared across all experiments and the same randomly chosen initial point is used in all experiments. The stochastic differential equations are simulated using Euler-Maruyama scheme with a time-step of $\tau^2$, starting from the second iterate of the discrete dynamics, which are computed deterministically. While for the ordinary differential equation Euler method is used.\\
We use the following regularization function:
\[
\phi(x) = \begin{cases}
    x & \text{if } x>c\\
    c & \text{if }x\leq c
\end{cases}
\]
with threshold $c=\tau$.\\
We generated randomly a definite matrix $H$ of dimension 6 and in each run of the first-order approximation, the number of simulations is set to
\[
\min\left\{ \max\left\{ 100\sqrt{T}\,\tau^{-2},\ 10^4 \right\},\ 10^7 \right\},
\]
while for the second-order approximation we use
\[
\min\left\{ \max\left\{ 100\sqrt{T}\,\tau^{-4},\ 10^5 \right\},\ 10^7 \right\},
\]
where \( T = 10 \). The values of $\epsilon$ and $\sigma$ are the same.\\
The parameter $\epsilon$ is kept constant at $10^{-6}$ and $\sigma$ is fixed to $1$.
All experiments were conducted on a computing node equipped with a NVIDIA A40 GPU (46 GB memory) and an AMD EPYC 7763 64-core CPU.\\

\subsection*{Acknowledgments}
{M.~R.} acknowledges the partial support of the project PNRR - M4C2 - Investimento 1.3, Partenariato Esteso PE00000013 - \emph{FAIR - Future Artificial Intelligence Research} - Spoke 1 \emph{Human-centered AI}, funded by the European Commission under the NextGeneration EU programme, of the project \emph{Noise in fluid dynamics and related models} funded by the MUR Progetti di Ricerca di Rilevante Interesse Nazionale (PRIN) Bando 2022 - grant 20222YRYSP, of the project \emph{APRISE - Analysis and Probability in Science} funded by the the University of Pisa, grant PRA\_2022\_85, and of the MUR Excellence Department Project awarded to the Department of Mathematics, University of Pisa, CUP I57G22000700001.
\bibliographystyle{plain}
\bibliography{ref}

\end{document}